\documentclass{article}

\usepackage{authblk}

\usepackage[top=30truemm,bottom=30truemm,left=25truemm,right=25truemm]{geometry}
\usepackage[square,numbers]{natbib}

\usepackage{hyperref}
\usepackage{url}

\usepackage{amsmath}
\usepackage{amssymb}
\usepackage{amsthm}
\usepackage{bm}
\usepackage{stmaryrd}

\usepackage{graphicx} 

\usepackage[subrefformat=parens]{subcaption}

\usepackage{algorithm}

\usepackage{booktabs}

\usepackage{hyperref}
\usepackage{url}

\usepackage{multirow}

\newtheorem{theorem}{Theorem}[section]
\newtheorem{proposition}[theorem]{Proposition}
\newtheorem{lemma}[theorem]{Lemma}

\usepackage{algorithmic} 

\usepackage{cancel}
\allowdisplaybreaks

\usepackage{listings}

\usepackage{xcolor}

\definecolor{codegreen}{rgb}{0,0.6,0}
\definecolor{codegray}{rgb}{0.5,0.5,0.5}
\definecolor{codepurple}{rgb}{0.58,0,0.82}
\definecolor{backcolour}{rgb}{0.95,0.95,0.92}

\lstdefinestyle{mystyle}{
  backgroundcolor=\color{backcolour}, commentstyle=\color{codegreen},
  keywordstyle=\color{magenta},
  numberstyle=\tiny\color{codegray},
  stringstyle=\color{codepurple},
  basicstyle=\ttfamily\scriptsize,
  breakatwhitespace=false,
  breaklines=true,
  captionpos=b,
  keepspaces=true,
  numbers=left,
  numbersep=5pt,
  showspaces=false,
  showstringspaces=false,
  showtabs=false,
  tabsize=2
}
\theoremstyle{definition}

\title{A Probabilistic Formulation of Offset Noise in Diffusion Models}
\author[$\ast$]{Takuro Kutsuna}
\affil[$\ast$]{\normalsize Toyota Central R\&D Labs., Inc.}
\date{}

\begin{document}
\maketitle
\begin{abstract}
  Diffusion models have become fundamental tools for modeling data distributions in machine learning. Despite their success, these models face challenges when generating data with extreme brightness values, as evidenced by limitations observed in practical large-scale diffusion models. Offset noise has been proposed as an empirical solution to this issue, yet its theoretical basis remains insufficiently explored. In this paper, we propose a novel diffusion model that naturally incorporates additional noise within a rigorous probabilistic framework. Our approach modifies both the forward and reverse diffusion processes, enabling inputs to be diffused into Gaussian distributions with arbitrary mean structures. We derive a loss function based on the evidence lower bound and show that the resulting objective is structurally analogous to that of offset noise, with time-dependent coefficients. Experiments on controlled synthetic datasets demonstrate that the proposed model mitigates brightness-related limitations and achieves improved performance over conventional methods, particularly in high-dimensional settings.
\end{abstract}

\section{Introduction} \label{sec:introduction}

One of the primary objectives of statistical machine learning is to model data distributions, a task that has supported recent advancements in generative artificial intelligence. The goal is to estimate a model that approximates an unknown distribution on the basis of multiple samples drawn from it. For example, when the data consists of images, the estimated model can be used to generate synthetic images that follow the same distribution.

Diffusion models \cite{sohl2015deep,ho2020denoising,songscore,karras2022elucidating} have emerged as powerful tools for estimating probability distributions and generating new data samples.
They have been shown to outperform other generative models, such as generative adversarial networks (GANs) \cite{goodfellow2014generative}, particularly in image generation tasks \cite{dhariwal2021diffusion}.
Due to their flexibility and effectiveness, diffusion models are now employed in a wide range of applications, including drug design \cite{corso2023diffdock,guan2023decompdiff}, audio synthesis \cite{kongdiffwave}, and text generation \cite{austin2021structured,li2022diffusion}.

A well-known challenge faced by diffusion models for image generation is their difficulty in producing images with extremely low or high brightness across the entire image \cite{offset_noise,lin2024common,hu2024one}. For example, it has been reported that Stable Diffusion \cite{rombach2022high}, a popular diffusion model for text-conditional image generation, struggles to generate fully black or fully white images when given prompts such as "Solid black image" or "A white background" \cite{lin2024common}.\footnote{The study in \cite{lin2024common} uses Stable Diffusion 2.1-base.}

Offset noise \cite{offset_noise} has been proposed as a solution to this issue and has been empirically demonstrated to be effective; however, its theoretical foundation remains unclear.
Specifically, offset noise introduces additional noise~$\bm{\epsilon}_c \sim q(\bm{\epsilon}_c)$, which is correlated across image channels, into the standard normal noise used during the training of denoising diffusion models \cite{ho2020denoising}.
Experiments have demonstrated that offset noise effectively mitigates brightness-related issues, and this technique has been incorporated in widely used models, such as SDXL \cite{podell2024sdxl}, a successor to Stable Diffusion.
Nevertheless, the theoretical justification for introducing~$\bm{\epsilon}_c$ during training remains ambiguous, raising concerns that the use of offset noise may diverge from the well-established theoretical framework of the original diffusion models.\footnote{For example, Lin et al. \cite{lin2024common} states that ``(offset noise) is incongruent with the theory of the diffusion process,'' while Hu et al. \cite{hu2024one} refers to offset noise as ``an unprincipled ad hoc adjustment.''}

In this study, we propose a novel diffusion model whose training loss function, derived from the evidence lower bound (ELBO), takes a similar form to the loss function with offset noise, with certain adjustments.
The proposed model modifies the forward and reverse processes of the original discrete-time diffusion models~\cite{sohl2015deep,ho2020denoising} to naturally incorporate additional noise~$\bm{\xi} \sim q(\bm{\xi})$, which corresponds to~$\bm{\epsilon}_c$ in offset noise.
The key difference between the loss function of the proposed model and that of the offset noise model lies in the treatment of the additional noise. In the proposed model, the noise is multiplied by time-dependent coefficients before being added to the standard normal noise~$\bm{\epsilon}$.
In contrast to offset noise, the proposed model is grounded in a well-defined probabilistic framework, ensuring theoretical compatibility with other methods for diffusion models. In particular, we explore its integration with the $v$-prediction framework \cite{salimans2022progressive}.

Another feature of the proposed model is that, unlike conventional diffusion models, which diffuse any input into standard Gaussian noise with zero mean, the proposed model diffuses any input into Gaussian noise with mean $\bm{\xi}$, where $\bm{\xi} \sim q(\bm{\xi})$.
In the reverse process, a new sample is generated starting from Gaussian noise with the same mean $\bm{\xi}$.
Since the distribution~$q(\bm{\xi})$ can be specified as an arbitrary distribution, the proposed model allows inputs to be diffused into a Gaussian distribution with any desired mean structure and generates new samples from that distribution.
If we set~$q(\bm{\xi})$ as a Dirac delta function at $\bm{\xi}=0$, the proposed model reduces to the conventional diffusion model, indicating that it includes the original diffusion models as a special case.

In summary, the contributions of this study are as follows:
\begin{itemize}
  \item We construct a probabilistically consistent diffusion model with an auxiliary random variable~$\bm{\xi}$, whose ELBO yields a loss function structurally similar to that of the offset noise model. While the ELBO derivation follows the standard procedure once the model is specified, establishing such a model itself is nontrivial. The key difference between the two loss functions is that, in the proposed model, the additional noise is scaled by time-dependent coefficients before being added to the standard normal noise. (Proposition~\ref{prop:loss_proposed})
  \item The proposed model generalizes conventional diffusion models by diffusing its inputs into Gaussian distributions with arbitrary mean structures, including the original zero-mean Gaussian distribution as a special case. (Proposition~\ref{prop:fwd})
  \item Because the proposed model is grounded in a well-defined probabilistic framework, in contrast to the offset noise model, it ensures theoretical compatibility with other methods for diffusion models. In particular, we discuss its integration with $v$-prediction~\cite{salimans2022progressive}. (Section~\ref{sec:vpred})
  \item We provide a mathematical analysis of the average-brightness statistic associated with extreme-brightness behavior. In the terminal regime, the standard diffusion model concentrates this statistic around zero, with standard deviation of order~$O(n^{-1/2})$, whereas in the proposed model it converges to a non-degenerate distribution determined by~$q(\bm{\xi})$. This explains why the proposed method is advantageous in high-dimensional settings. (Proposition~\ref{prop:brightness_stat})
  \item We empirically demonstrate the superiority of the proposed model by using a synthetic dataset that simulates a scenario where image brightness is uniformly distributed from solid black and pure white. This scenario is shown to be less effectively modeled by conventional diffusion models, especially in high-dimensional data settings,  whereas the proposed model successfully generates data that follows the true distribution. (Section~\ref{sec:experiments})
\end{itemize}

\section{Preliminary} \label{sec:preliminary}
This section briefly reviews the conventional discrete-time diffusion model and the offset noise heuristic relevant to our formulation.

\subsection{Diffusion models}
Diffusion models learn a data distribution by defining a forward noising process and a reverse denoising process. We focus on the standard discrete-time formulation~\cite{sohl2015deep,ho2020denoising}, which also provides the variational interpretation used in this paper.

\subsubsection{Forward and reverse processes} \label{sec:ordinary_fw_rev}
Let~$\bm{x}_0 \in \mathbb{R}^n$ denote a data sample.\footnote{Although image data have spatial and channel structure, we treat them as vectors for notational simplicity.}
A standard diffusion model defines
\begin{align}
  q(\bm{x}_{1:T} | \bm{x}_0) & = \prod_{t=1}^T q(\bm{x}_t | \bm{x}_{t-1}), \label{eq:fp1}                                                                            \\
  q(\bm{x}_t | \bm{x}_{t-1}) & = \mathcal{N}\left(\bm{x}_t  \ \middle| \ \sqrt{1 - \beta_t} \bm{x}_{t-1}, \beta_t I\right) \text{ for } t=1,\ldots,T, \label{eq:fp2}
\end{align}
where~$\beta_t > 0$ is a prescribed variance schedule. As~$t$ increases, the forward process gradually destroys information in~$\bm{x}_0$ so that~$\bm{x}_T$ approaches standard Gaussian noise.
The reverse process is defined as
\begin{align}
  p_\theta(\bm{x}_{0:T})            & = p(\bm{x}_T) \prod_{t=1}^T p_\theta(\bm{x}_{t-1} | \bm{x}_{t}),        \label{eq:rev1}                                                 \\
  p(\bm{x}_T)                       & = \mathcal{N}\left(\bm{x}_T \ \middle| \ 0,  I\right),  \label{eq:rev2}                                                                 \\
  p_\theta(\bm{x}_{t-1} | \bm{x}_t) & = \mathcal{N}\left(\bm{x}_{t-1} \ \middle| \ \mu_\theta(\bm{x}_t, t),  \sigma_t^2 I\right) \text{ for } t=1,\ldots, T,  \label{eq:rev3}
\end{align}
where~$\mu_\theta$ is a neural network that predicts the mean of $\bm{x}_{t-1}$. Following common practice, we treat~$\sigma_t^2$ as fixed rather than as a learnable parameter, typically setting~$\sigma_t^2=\beta_t$~\cite{ho2020denoising}.

The parameter~$\theta$ is learned by maximizing the evidence lower bound (ELBO) of the log-likelihood:
\begin{align}
  \log p_\theta(\bm{x}_0) & \geq \mathbb{E}_{q(\bm{x}_{1:T} | \bm{x}_0)}\left[\log \frac{p_\theta(\bm{x}_{0:T})}{q(\bm{x}_{1:T} | \bm{x}_0)}\right]. \label{eq:elbo}
\end{align}

\subsubsection{Denoising modeling} \label{sec:ddpm}
Instead of directly predicting the mean of $\bm{x}_{t-1}$ with~$\mu_\theta$, DDPM~\cite{ho2020denoising} parameterizes~$\mu_\theta$ as
\begin{align}
  \mu_\theta(\bm{x}_t,t) = \frac{1}{\sqrt{\alpha_t}} \bm{x}_t - \frac{1 - \alpha_t}{\sqrt{1 - \bar{\alpha}_t} \sqrt{\alpha_t}} \epsilon_\theta (\bm{x}_t, t),  \label{eq:mu_epsilon}
\end{align}
where~$\alpha_t$ and~$\bar{\alpha}_t$ are determined by the noise schedule~$\beta_t$. Under this parameterization, maximizing the ELBO leads to the following simplified noise prediction loss, with the time-dependent weighting omitted:
\begin{align}
  \hat{\ell}_\text{simple}(\theta; \bm{x}_0) & = \mathbb{E}_{\mathcal{U}(t | 1,T), \mathcal{N}(\bm{\epsilon}_0 | 0, I)} \left[\left\|\bm{\epsilon}_0  - \epsilon_\theta\left(\sqrt{\bar{\alpha}_t} \bm{x}_0 + \sqrt{1 - \bar{\alpha}_t} \bm{\epsilon}_0, t\right)\right\|^2 \right], \label{eq:loss_eps_simple}
\end{align}
where~$\mathcal{U}(t | 1, T)$ denotes the discrete uniform distribution over~$\left\{1, \ldots, T\right\}$.

\subsection{Offset noise} \label{sec:offsetnoise}
Standard diffusion models often underrepresent images with extremely low or high global brightness~\cite{offset_noise,lin2024common,hu2024one}. Offset noise~\cite{offset_noise} addresses this issue by augmenting the standard Gaussian noise with an additional correlated component during training:
\begin{align}
  \hat{\ell}_\text{offset}(\theta; \bm{x}_0) & = \mathbb{E}_{\mathcal{U}(t | 1,T), \mathcal{N}(\bm{\epsilon}_0 | 0, I), q(\bm{\epsilon}_c)} \left[\left\|\bm{\epsilon}_0 + \bm{\epsilon}_c  - \epsilon_\theta\left(\sqrt{\bar{\alpha}_t} \bm{x}_0 + \sqrt{1 - \bar{\alpha}_t} \left(\bm{\epsilon}_0 + \bm{\epsilon}_c\right), t\right)\right\|^2 \right], \label{eq:loss_offsetnoise}
\end{align}
where~$q(\bm{\epsilon}_c)$ is a zero-mean normal distribution with fully correlated covariance across image channels. Formally,~$q(\bm{\epsilon}_c)$ is expressed as~$q(\bm{\epsilon}_c) = \mathcal{N}(\bm{\epsilon}_c \ | \ 0, \sigma_c^2 \Sigma_c)$, where~$\Sigma_c$ is a block-diagonal matrix whose entries are all ones within each channel, and~$\sigma_c^2$ controls the magnitude of the offset noise.

Empirically, this heuristic improves the generation of images with low or high brightness and has been adopted in practical systems such as SDXL~\cite{podell2024sdxl}. However, it is introduced directly at the loss level and does not specify the corresponding forward and reverse probabilistic processes. This gap motivates the probabilistic reformulation developed in the next section.

\section{Proposed model} \label{sec:proposed}
We first define the forward and reverse processes of the proposed model and derive the corresponding ELBO-based loss function. We show that the resulting loss takes a form similar to that of the offset noise model, differing only in the coefficients of the additional noise. While the algebraic decomposition of the ELBO follows the standard derivation once the model is specified, the key point is that the proposed latent-variable diffusion process yields a tractable ELBO whose resulting objective has an offset-noise-like form.

\subsection{Forward and reverse processes} \label{sec:prop_fw_rev}
The forward process in the proposed model is defined as follows:
\begin{align}
  q(\bm{x}_{1:T}, \bm{\xi} | \bm{x}_0) & = q(\bm{\xi}) \prod_{t=1}^T q(\bm{x}_t | \bm{x}_{t-1}, \bm{\xi}),   \label{eq:fp_proposed1}                                                                                                   \\
  q(\bm{x}_t | \bm{x}_{t-1}, \bm{\xi}) & = \mathcal{N}\left(\bm{x}_t  \ \middle| \ \sqrt{1 - \beta_t} \left( \bm{x}_{t-1} + \gamma_t \bm{\xi} \right), \beta_t \sigma_0^2 I\right) \text{ for } t=1,\ldots,T,  \label{eq:fp_proposed2}
\end{align}
where~$\bm{\xi} \in \mathbb{R}^n$ is an additional random variable with distribution~$q(\bm{\xi})$, independent of time~$t$. We do not impose a specific form on~$q(\bm{\xi})$, allowing it to be an arbitrary distribution. A scalar parameter~$\sigma_0 \in \mathbb{R}$ is introduced as a scaling factor for the variance. Additionally,~$\gamma_t \in \mathbb{R} \ (t=1, \ldots, T)$ denotes a coefficient of~$\bm{\xi}$ that determines the contribution of the additional noise in the loss function, as discussed in the next section. The construction of~$\gamma_t$ is described in Section~\ref{sec:gamma_t}.

The reverse process in the proposed model is defined as follows:
\begin{align}
  p_\theta(\bm{x}_{0:T}, \bm{\xi})  & = p(\bm{\xi}) p(\bm{x}_T | \bm{\xi}) \prod_{t=1}^T p_\theta(\bm{x}_{t-1} | \bm{x}_{t}), \label{eq:rev_prop1}                                \\
  p(\bm{\xi})                       & = q(\bm{\xi}),   \label{eq:rev_prop2}                                                                                                       \\
  p(\bm{x}_T | \bm{\xi})            & = \mathcal{N}\left(\bm{x}_T \ \middle| \ \bm{\xi}, \sigma_0^2 I\right),  \label{eq:rev_prop3}                                               \\
  p_\theta(\bm{x}_{t-1} | \bm{x}_t) & = \mathcal{N}\left(\bm{x}_{t-1} \ \middle| \ \mu_\theta(\bm{x}_t, t), \sigma_t^2 I\right) \text{ for } t=1,\ldots, T.  \label{eq:rev_prop4}
\end{align}
The key difference from the standard reverse process in \eqref{eq:rev1}--\eqref{eq:rev3} is that $\bm{x}_T$ follows a Gaussian distribution with mean $\bm{\xi}$ rather than zero. The transition distribution $p_\theta(\bm{x}_{t-1} | \bm{x}_t)$ in \eqref{eq:rev_prop4} is identical to that in \eqref{eq:rev3}.

\subsection{Loss function for the proposed model} \label{sec:main_results}
We define~$\alpha_0 = 1$,~$\alpha_t = 1 - \beta_t \ (t=1, \ldots, T)$, and~$\bar{\alpha}_t = \prod_{i=0}^t \alpha_i$.\footnote{In standard diffusion models \cite{ho2020denoising},~$\alpha_0$ is not defined, but here we introduce~$\alpha_0 = 1$ for convenience in our derivations. Consequently, the definition of~$\bar{\alpha}_t$ differs from the conventional one~($\prod_{i=1}^t \alpha_i$); however, since~$\alpha_0 = 1$, this modified~$\bar{\alpha}_t$ is essentially equivalent to the standard~$\bar{\alpha}_t$.}
Given the forward and reverse processes defined in the previous section, the training loss is derived from the ELBO.
\begin{proposition}[Training loss function]\label{prop:loss_proposed}
  Suppose the forward process is defined as in \eqref{eq:fp_proposed1} and \eqref{eq:fp_proposed2}, and the reverse process as in \eqref{eq:rev_prop1}--\eqref{eq:rev_prop4}. Then, the loss function that maximizes the ELBO of~$\log p_\theta(\bm{x}_0)$ is
  \begin{align}
    \ell(\theta; \bm{x}_0) = \mathbb{E}_{q(\bm{\xi}), \mathcal{U}(t | 1,T), \mathcal{N}(\bm{\epsilon}_0 | 0, I)} \left[\lambda_t \left\|\sigma_0 \bm{\epsilon}_0 + \phi_t \bm{\xi} - \epsilon_\theta \left(\sqrt{\bar{\alpha}_t} \bm{x}_0 + \sqrt{1 - \bar{\alpha}_t} \left(\sigma_0 \bm{\epsilon}_0 + \psi_t \bm{\xi}\right), t \right)\right\|^2 \right], \label{eq:loss_proposed}
  \end{align}
  where~$\lambda_t$ is given by
  \begin{align}
    \lambda_t = \frac{(1-\alpha_t)^2}{2 \sigma_t^2 \alpha_t (1-\bar{\alpha}_t)},
  \end{align}
  and~$\phi_t$ and~$\psi_t$ are given by
  \begin{align}
    \phi_t & = \frac{\sqrt{\alpha_t} \sqrt{1 - \bar{\alpha}_t}}{1-\alpha_t} \gamma_t \text{ for } t=1,\ldots, T, \label{eq:def_phi}                                      \\
    \psi_t & = \frac{1}{\sqrt{1 - \bar{\alpha}_t}} \sum_{i=1}^t \sqrt{\frac{\bar{\alpha}_t}{\bar{\alpha}_{i-1}}} \gamma_i \text{ for } t=1,\ldots, T. \label{eq:def_psi}
  \end{align}
\end{proposition}
In the following subsections, we provide a detailed derivation of Proposition~\ref{prop:loss_proposed}.

\subsubsection{Evidence lower bound}
The ELBO can be decomposed into three terms:
\begin{align}
  \log p_\theta(\bm{x}_0) & \geq \mathbb{E}_{q(\bm{x}_{1:T}, \bm{\xi} | \bm{x}_0)}\left[\log \frac{p_\theta(\bm{x}_{0:T}, \bm{\xi})}{q(\bm{x}_{1:T}, \bm{\xi} | \bm{x}_0)}\right]                 \notag                                                                                                                                                                 \\
  & = \underbrace{\mathbb{E}_{q(\bm{\xi})} \mathbb{E}_{q(\bm{x}_{1} | \bm{x}_0, \bm{\xi})}\left[\log p_\theta(\bm{x}_0 | \bm{x}_1)\right]}_{\mathcal{L}_1} \underbrace{- \mathbb{E}_{q(\bm{\xi})} \left[ D_{\text{KL}} \left(q(\bm{x}_T | \bm{x}_0, \bm{\xi}) \ || \ p(\bm{x}_T | \bm{\xi})\right) \right]}_{\mathcal{L}_2} \notag \\
  & \qquad \underbrace{- \sum_{t=2}^T \mathbb{E}_{q(\bm{\xi})} \mathbb{E}_{q(\bm{x}_{t} | \bm{x}_0, \bm{\xi})} \left[ D_{\text{KL}}\left(q(\bm{x}_{t-1} | \bm{x}_{t}, \bm{x}_0, \bm{\xi}) \ || \ p_\theta(\bm{x}_{t-1}|\bm{x}_t) \right)\right]}_{\mathcal{L}_3}, \label{eq:elbo_prop}
\end{align}
where~$D_\text{KL}(\cdot \,||\, \cdot)$ denotes the Kullback--Leibler (KL) divergence. A detailed derivation of \eqref{eq:elbo_prop} is provided in Appendix~\ref{apdx:elbo_prop_derivation}. We denote the three terms by~$\mathcal{L}_1$,~$\mathcal{L}_2$, and~$\mathcal{L}_3$, respectively, and analyze them in the order~$\mathcal{L}_2$,~$\mathcal{L}_3$, and~$\mathcal{L}_1$.

The decomposition in \eqref{eq:elbo_prop} itself closely parallels the standard variational derivation for diffusion models. The nontrivial point is that, after introducing~$\bm{\xi}$ into every forward transition and into the terminal distribution, all resulting conditional distributions remain analytically tractable. This leads to closed-form expressions for~$q(\bm{x}_t | \bm{x}_0, \bm{\xi})$ and~$q(\bm{x}_{t-1} | \bm{x}_t, \bm{x}_0, \bm{\xi})$, and hence to the coefficients~$\phi_t$ and~$\psi_t$ that determine precisely how the proposed objective differs from offset noise and from standard diffusion training.

\subsubsection{The $\mathcal{L}_2$ term}
Since~$\mathcal{L}_2$ does not depend on~$\theta$, it can be ignored during optimization. The value of~$\mathcal{L}_2$ increases as the distribution of~$\bm{x}_T$ induced by the forward process becomes closer to that of the reverse process.
It can be shown that these distributions coincide under appropriate choices of~$\beta_t$ and~$\gamma_t$ (see Proposition~\ref{prop:fwd}). Under such conditions,~$\mathcal{L}_2$ attains its optimal value of zero.

\subsubsection{Simplifying the $\mathcal{L}_3$ term}
\paragraph{Derivation of the forward conditional distribution}
The variable~$\bm{x}_t \ (t=1, \ldots, T)$ that follows~$q(\bm{x}_{t} | \bm{x}_0, \bm{\xi})$ in~$\mathcal{L}_3$ can be expressed as
\begin{align}
  \bm{x}_t & = \sqrt{\bar{\alpha}_t} \bm{x}_0 + \sqrt{1 - \bar{\alpha}_t} \left(\sigma_0 \bm{\epsilon}_0 + \psi_t \bm{\xi}\right), \label{eq:xt_x0_ep0_f}
\end{align}
where~$\bm{\epsilon}_0 \sim \mathcal{N}(\bm{\epsilon}_0 \ | \ 0,I)$ and $\psi_t\ (t=1,\ldots,T)$ is given by \eqref{eq:def_psi}.
A detailed derivation of \eqref{eq:xt_x0_ep0_f} is provided in Appendix~\ref{apdx:xt_x0_ep0_f_derivation}.
From \eqref{eq:xt_x0_ep0_f}, the conditional distribution of~$\bm{x}_t$ given~$\bm{x}_0$ and~$\bm{\xi}$ is
\begin{align}
  q(\bm{x}_t | \bm{x}_0, \bm{\xi} ) & = \mathcal{N}\left(\bm{x}_t \ \middle| \ \sqrt{\bar{\alpha}_t} \bm{x}_0  + \sqrt{1 - \bar{\alpha}_t} \psi_t \bm{\xi}, \left(1 - \bar{\alpha}_t\right) \sigma_0^2 I\right).  \label{eq:q_xt_x0_xi_norm}
\end{align}
From \eqref{eq:q_xt_x0_xi_norm}, the following proposition holds:
\begin{proposition} \label{prop:fwd}
  Suppose~$\bar{\alpha}_t \to 0$ and~$\psi_t \to 1$ as~$t \to T$. Then,
  \begin{align}
    q(\bm{x}_t | \bm{x}_0, \bm{\xi}) \;\to\; q(\bm{x}_T | \bm{x}_0, \bm{\xi}) = \mathcal{N}(\bm{x}_T \mid \bm{\xi}, \sigma_0^2 I).
  \end{align}
\end{proposition}
Proposition~\ref{prop:fwd} shows that, in the proposed model, any input~$\bm{x}_0$ diffuses into a Gaussian distribution with mean~$\bm{\xi}$ and variance~$\sigma_0^2 I$ at the final time step.

\paragraph{Derivation of the reverse conditional distribution}
The conditional distribution~$q(\bm{x}_{t-1} | \bm{x}_{t}, \bm{x}_0, \bm{\xi}) \ (t=2, \ldots, T)$ in~$\mathcal{L}_3$ is given by
\begin{align}
  q(\bm{x}_{t-1} | \bm{x}_{t}, \bm{x}_0, \bm{\xi}) & =  \mathcal{N}\left(\bm{x}_{t-1} \ \middle| \ \tilde{\mu}_t(\bm{x}_t, \bm{x}_0, \bm{\xi}), \tilde{\beta}_t I\right),                                                                                    \label{eq:q_x_t-1_x_0_xi} \\
  \tilde{\mu}_t(\bm{x}_t, \bm{x}_0, \bm{\xi})      & = \frac{\sqrt{\alpha_t} (1-\bar{\alpha}_{t-1})}{1 - \bar{\alpha}_t} \bm{x}_t + \frac{(1 - \alpha_t)\sqrt{\bar{\alpha}_{t-1}}}{1 - \bar{\alpha}_t} \bm{x}_0 + \nu_t \bm{\xi},      \label{eq:tilde_mu_x0}                          \\
  \tilde{\beta}_t                                  & = \frac{(1 - \alpha_t)(1 - \bar{\alpha}_{t-1})}{1 - \bar{\alpha}_t} \sigma_0^2, \label{eq:tilde_beta_t}
\end{align}
where~$\nu_t \ (t=2,\ldots,T)$ is defined as
\begin{align}
  \nu_t = \frac{(1 - \alpha_t)\sqrt{1 - \bar{\alpha}_{t-1}} \psi_{t-1} - \alpha_t (1-\bar{\alpha}_{t-1})\gamma_t}{1 - \bar{\alpha}_t}. \label{eq:nu_t}
\end{align}
A detailed derivation of \eqref{eq:q_x_t-1_x_0_xi}--\eqref{eq:nu_t} is provided in Appendix~\ref{apdx:q_xt_x0_xi_derivation}.

When~$\gamma_t$ is constructed as described in Section~\ref{sec:gamma_t}, we have~$\nu_t = 0 \ (t=2, \ldots, T)$. Under this condition, and assuming~$\sigma_0 = 1$, the conditional distribution reduces to~$q(\bm{x}_{t-1} | \bm{x}_{t}, \bm{x}_0)$ in standard diffusion models~\cite{ho2020denoising} (see Proposition~\ref{prop:nu}).

\paragraph{Towards denoising formulation}
To apply the denoising approach~\cite{ho2020denoising} to the proposed model, we must first establish the following lemma:
\begin{lemma} \label{lemma:mu}
  The quantity~$\tilde{\mu}_t(\bm{x}_t, \bm{x}_0, \bm{\xi}) \ (t=2, \ldots, T)$ in \eqref{eq:tilde_mu_x0} can be rewritten as
  \begin{align}
    \tilde{\mu}_t(\bm{x}_t, \bm{x}_0, \bm{\xi}) & = \frac{1}{\sqrt{\alpha_t}} \bm{x}_t - \frac{1 - \alpha_t}{\sqrt{1 - \bar{\alpha}_t} \sqrt{\alpha_t}} \left(\sigma_0 \bm{\epsilon}_0 + \phi_t \bm{\xi}\right),  \label{eq:mu_x0_to_e}
  \end{align}
  where~$\bm{\epsilon}_0 \sim \mathcal{N}(\bm{\epsilon}_0 \ | \ 0, I)$ and~$\phi_t \ (t=2, \ldots, T)$ is given by \eqref{eq:def_phi}.
\end{lemma}
\begin{proof}
  See Appendix~\ref{apdx:proof_lemma_mu}.
\end{proof}
Instead of directly predicting~$\tilde{\mu}_t(\bm{x}_t, \bm{x}_0, \bm{\xi})$, we parameterize~$\mu_\theta(\bm{x}_t, t)$ as in \eqref{eq:mu_epsilon}, following~\cite{ho2020denoising}. Under this parameterization,~$\epsilon_\theta(\bm{x}_t, t)$ becomes the training target instead of~$\mu_\theta(\bm{x}_t, t)$.

Using Lemma~\ref{lemma:mu}, the KL divergence in~$\mathcal{L}_3$ can be rewritten as
\begin{align}
  D_{\text{KL}}\left(q(\bm{x}_{t-1} | \bm{x}_{t}, \bm{x}_0, \bm{\xi}) \ || \ p_\theta(\bm{x}_{t-1}|\bm{x}_t) \right) = \lambda_t \mathbb{E}_{\mathcal{N}(\bm{\epsilon}_0 | 0, I)} \left[\left\|\sigma_0 \bm{\epsilon}_0 + \phi_t \bm{\xi} - \epsilon_\theta(\bm{x}_t, t)\right\|^2 \right] + C_1, \label{eq:L3_KLD}
\end{align}
where~$C_1$ is a constant independent of~$\theta$.

\subsubsection{Simplifying the $\mathcal{L}_1$ term} \label{sec:L1}
From \eqref{eq:def_phi} and \eqref{eq:def_psi}, we have~$\phi_1 = \psi_1$. The expectation in~$\mathcal{L}_1$ can then be written as
\begin{align}
  \mathbb{E}_{q(\bm{x}_{1} | \bm{x}_0, \bm{\xi})}\left[\log p_\theta(\bm{x}_0 | \bm{x}_1) \right] & = - \lambda_1 \mathbb{E}_{\mathcal{N}(\bm{\epsilon}_0 | 0, I)} \left[ \left\|\sigma_0 \bm{\epsilon}_0 + \phi_1 \bm{\xi} - \epsilon_\theta (\bm{x}_1, 1) \right\|^2 \right] + C_2, \label{eq:L1}
\end{align}
where~$C_2$ is a constant independent of~$\theta$. A detailed derivation of \eqref{eq:L1} is provided in Appendix~\ref{apdx:L1_derivation}.

\subsubsection{Derivation of the training loss function}
Combining \eqref{eq:L3_KLD} and \eqref{eq:L1}, the objective that maximizes the ELBO in \eqref{eq:elbo_prop} with respect to~$\theta$ is given by~$\ell(\theta; \bm{x}_0)$ in \eqref{eq:loss_proposed}. This completes the proof of Proposition~\ref{prop:loss_proposed}.

\subsection{Comparison with existing models}
Following~\cite{ho2020denoising}, we define a simplified version of~$\ell(\theta; \bm{x}_0)$ by setting all~$\lambda_t$ in \ref{eq:loss_proposed} to~$1$:
\begin{align}
  \ell_\text{simple}(\theta; \bm{x}_0) & = \mathbb{E}_{q(\bm{\xi}), \mathcal{U}(t | 1,T), \mathcal{N}(\bm{\epsilon}_0 | 0, I)} \left[\left\|\sigma_0 \bm{\epsilon}_0 + \phi_t \bm{\xi} - \epsilon_\theta \left(\sqrt{\bar{\alpha}_t} \bm{x}_0 + \sqrt{1 - \bar{\alpha}_t} \left(\sigma_0 \bm{\epsilon}_0 + \psi_t \bm{\xi}\right), t \right)\right\|^2 \right]. \label{eq:ell_simple}
\end{align}

\paragraph{Comparison with offset noise model}
The loss function of the offset noise model in \eqref{eq:loss_offsetnoise} is structurally similar to \eqref{eq:ell_simple}. The key difference is that, in the proposed model,~$\bm{\xi} \sim q(\bm{\xi})$ is added to~$\bm{\epsilon}_0$ with time-dependent coefficients~$\phi_t$ and~$\psi_t$, whereas in the offset noise model,~$\bm{\epsilon}_c \sim q(\bm{\epsilon}_c)$ is added with a constant coefficient independent of the time step. This difference arises from the fact that the proposed model is derived from a consistent probabilistic framework.

In particular, the proposed formulation specifies the terminal distribution, the posterior~$q(\bm{x}_{t-1} \mid \bm{x}_t, \bm{x}_0, \bm{\xi})$, and the time-dependent coefficients~$\phi_t$ and~$\psi_t$ in a unified manner through the forward and reverse processes. In contrast, simply augmenting the standard diffusion objective with an auxiliary expectation does not determine these quantities and therefore lacks a corresponding probabilistic interpretation.

The two models also differ in their reverse processes. In the proposed model,~$\bm{x}_T$ is initialized as Gaussian noise with mean~$\bm{\xi} \sim q(\bm{\xi})$ (see \eqref{eq:rev_prop3}), whereas in the offset noise model, the reverse process typically follows the standard diffusion formulation with zero-mean Gaussian initialization (see \eqref{eq:rev2}).

\paragraph{Comparison with existing diffusion models}
In conventional diffusion models (Section~\ref{sec:ordinary_fw_rev}), the forward process maps the input~$\bm{x}_0$ to a Gaussian distribution with zero mean and variance~$I$, and the reverse process starts from this standard Gaussian distribution. In contrast, as shown in Proposition~\ref{prop:fwd}, the proposed model maps~$\bm{x}_0$ to a Gaussian distribution with mean~$\bm{\xi}$ and variance~$\sigma_0^2 I$, and the reverse process is initialized from the same distribution, ensuring consistency between the forward and reverse processes. This consistency is also justified from the perspective of the $\mathcal{L}_2$ term in the ELBO, which measures the discrepancy between the terminal distributions of the forward and reverse processes, which vanishes when these distributions coincide.
If~$q(\bm{\xi})$ is chosen as a Dirac delta at zero and~$\sigma_0=1$, the proposed model reduces to the conventional diffusion model.
From this viewpoint, the proposed model generalizes the conventional model by replacing its terminal behavior with a controllable distribution induced by~$q(\bm{\xi})$. As a concrete example, choosing~$\bm{\xi}$ to represent an offset-noise-like component enables explicit control over the terminal behavior in the average-brightness direction. We make this connection precise in the next subsection.

\subsection{Theoretical analysis of extreme brightness via the average-brightness statistic} \label{sec:brightness_theory}

We consider the linear statistic
\begin{align}
  B_n(\bm{x}) := \frac{1}{n}\mathbf{1}_n^\top \bm{x},
\end{align}
which corresponds to the average brightness when $\bm{x} \in \mathbb{R}^n$ represents an image.

In this subsection, we specialize to
\begin{align}
  q(\bm{\xi}) = \mathcal{N}(\bm{\xi} \mid 0, \sigma_c^2 \mathbf{1}_{n \times n}), \label{eq:xi_1c_on}
\end{align}
where $\mathbf{1}_{n \times n}$ denotes the $n \times n$ matrix with all entries equal to~$1$. This is the single-channel analogue of the covariance used in offset noise. Under \eqref{eq:xi_1c_on},~$\bm{\xi}$ is supported on the one-dimensional subspace~$\mathrm{span}\{\mathbf{1}_n\}$, so the additional randomness acts only along the average-brightness direction.

\begin{proposition}[Dynamics of the average-brightness statistic] \label{prop:brightness_stat}
  Suppose~$q(\bm{\xi})$ is given by \eqref{eq:xi_1c_on}, and let~$z := B_n(\bm{\xi})$. Then~$z \sim \mathcal{N}(0, \sigma_c^2)$ and, under the proposed forward process,
  \begin{align}
    B_n(\bm{x}_t) = \sqrt{\bar{\alpha}_t} B_n(\bm{x}_0) + \sqrt{1 - \bar{\alpha}_t}\left(\sigma_0 \varepsilon_B + \psi_t z\right), \label{eq:brightness_prop}
  \end{align}
  where~$\varepsilon_B \sim \mathcal{N}(0, 1/n)$. Consequently,
  \begin{align}
    \mathrm{Var}\!\left[B_n(\bm{x}_t) \mid \bm{x}_0\right] = (1 - \bar{\alpha}_t)\left(\frac{\sigma_0^2}{n} + \psi_t^2 \sigma_c^2\right). \label{eq:brightness_var_prop}
  \end{align}
  In contrast, under the standard diffusion model,
  \begin{align}
    B_n(\bm{x}_t^\mathrm{std}) = \sqrt{\bar{\alpha}_t} B_n(\bm{x}_0) + \sqrt{1 - \bar{\alpha}_t} \varepsilon_B, \label{eq:brightness_std}
  \end{align}
  from which it follows that
  \begin{align}
    \mathrm{Var}\!\left[B_n(\bm{x}_t^\mathrm{std}) \mid \bm{x}_0\right] = (1 - \bar{\alpha}_t)\frac{1}{n}. \label{eq:brightness_var_std}
  \end{align}
\end{proposition}
\begin{proof}
  See Appendix~\ref{apdx:proof_brightness_stat}.
\end{proof}

The key difference is the source of randomness along the average brightness direction.
In the standard model, fluctuations come only from $\varepsilon_B$, whose variance scales as $1/n$. As a result, the average brightness of $\bm{x}_t$ becomes highly concentrated as the dimension increases.
In the proposed model, an additional term $\psi_t z$ introduces fluctuations of constant scale, preventing this concentration.

This difference has an important consequence.
If the data distribution exhibits~$O(1)$ variation in~$B_n(\bm{x}_0)$, then, in the standard model, near-terminal noisy states differ along this direction only at the~$O(n^{-1/2})$ scale. The reverse model must therefore reconstruct an~$O(1)$ signal from inputs whose separation in that coordinate is vanishingly small.
In other words, the model is required to map almost identical noisy states to substantially different clean signals along the average-brightness direction.
This scale mismatch makes denoising along the average-brightness direction challenging and amplifies approximation errors in the learned denoiser.
In contrast, in the proposed model, the term~$\psi_t z$ can preserve~$O(1)$ variability in the same direction as long as~$\psi_t$ remains bounded away from zero.
Consequently, near-terminal noisy states may remain distinguishable by their average brightness even in high dimensions, which may alleviate the difficulty of recoverying this component in the reverse process.

\section{Method for constructing $\gamma_t$ in the proposed model} \label{sec:gamma_t}
The coefficients~$\phi_t$ and~$\psi_t$ depend on both the variance schedule~$\beta_t$ (or equivalently~$\alpha_t$ and~$\bar{\alpha}_t$) and the sequence~$\gamma_t$, as shown in \eqref{eq:def_phi} and \eqref{eq:def_psi}. In this section, we treat~$\beta_t$ as given, for example by adopting a standard schedule used in diffusion models, and describe how to construct~$\gamma_t$ accordingly. For each admissible choice of~$\beta_t$, this construction induces the corresponding coefficients~$\phi_t$ and~$\psi_t$; it does not impose an additional restriction on the variance schedule itself.

\subsection{Noise-matching strategy}
In the loss function \eqref{eq:loss_eps_simple} of standard diffusion models, the noise added to~$\bm{x}_0$ and the target noise predicted by~$\epsilon_\theta$ are identical.
In contrast, in the proposed loss \eqref{eq:loss_proposed}, the noise added to~$\bm{x}_0$ is~$\sigma_0 \bm{\epsilon}_0 + \psi_t \bm{\xi}$, whereas the target noise is~$\sigma_0 \bm{\epsilon}_0 + \phi_t \bm{\xi}$.
To preserve the structure of the original loss, it is natural to impose the condition~$\psi_t = \phi_t$, so that the prediction target matches the injected noise, as in standard diffusion training. We refer to this choice of~$\gamma_t$ as the noise-matching strategy. The construction procedure is described below.

Fix a schedule~$\{\beta_t\}_{t=1}^T$ with~$0 < \beta_t < 1$, and hence~$0 < \alpha_t < 1$. Imposing~$\phi_t = \psi_t$ for~$t = 2, \ldots, T$ and substituting \eqref{eq:def_phi} and \eqref{eq:def_psi} yields
\begin{align*}
  \frac{\sqrt{\alpha_t} \sqrt{1 - \bar{\alpha}_t}}{1-\alpha_t} \gamma_t = \frac{1}{\sqrt{1 - \bar{\alpha}_t}} \sum_{i=1}^t \sqrt{\frac{\bar{\alpha}_t}{\bar{\alpha}_{i-1}}} \gamma_i.
\end{align*}
Rearranging this equation gives the following recursion for~$\gamma_t$:
\begin{align}
  \gamma_t = \frac{(1 - \alpha_t) \sqrt{\bar{\alpha}_{t-1}}}{\alpha_t (1-\bar{\alpha}_{t-1})} \sum_{i=1}^{t-1} \frac{\gamma_i}{\sqrt{\bar{\alpha}_{i-1}}}. \label{eq:phi_eq_psi}
\end{align}
Moreover, from Section~\ref{sec:L1}, we have
\begin{align*}
  \phi_1 = \psi_1 = \frac{\sqrt{\alpha_1}}{\sqrt{1-\alpha_1}} \gamma_1.
\end{align*}
Therefore, for any fixed schedule~$\{\beta_t\}_{t=1}^T$, defining~$\gamma_t \ (t=2, \ldots, T)$ recursively by \eqref{eq:phi_eq_psi} ensures that~$\phi_t = \psi_t \ (t=1, \ldots, T)$, independently of the choice of~$\gamma_1$, due to the linearity of the recursion. In this sense, the noise-matching strategy maps a given~$\beta_t$ schedule to the induced coefficients~$\gamma_t$,~$\phi_t$, and~$\psi_t$.

In the noise-matching strategy,~$\gamma_1$ is chosen so that the condition~$\psi_T = 1$ in Proposition~\ref{prop:fwd} is satisfied. Notably, the recursion \eqref{eq:phi_eq_psi} admits a scaling property: if~$\gamma_1$ is scaled by a positive constant~$C (>0)$, then the resulting sequences~$\gamma_t \ (t \geq 2)$, as well as~$\phi_t$ and~$\psi_t \ (t \geq 1)$, are all scaled by~$C$.
Based on this property, we first set~$\hat{\gamma}_1 = 1$ and compute~$\hat{\gamma}_t \ (t \geq 2)$ recursively using \eqref{eq:phi_eq_psi}. We then compute~$\hat{\psi}_T$ from \eqref{eq:def_psi} and define
\begin{align*}
  \gamma_t = \frac{\hat{\gamma}_t}{\hat{\psi}_T}.
\end{align*}
This normalization ensures that~$\psi_T = 1$.

The noise-matching strategy is summarized in Algorithm~\ref{alg:noise_matching_strategy}.
\begin{algorithm}
  \caption{Noise-matching strategy for constructing $\gamma_t$}
  \label{alg:noise_matching_strategy}
  \begin{algorithmic}[1]
    \STATE{$\hat{\gamma}_1 \gets 1$}
    \FOR{$t = 2$ to $T$}
    \STATE{$\hat{\gamma}_t \gets \frac{(1 - \alpha_t) \sqrt{\bar{\alpha}_{t-1}}}{\alpha_t (1-\bar{\alpha}_{t-1})} \sum_{i=1}^{t-1} \frac{\hat{\gamma}_i}{\sqrt{\bar{\alpha}_{i-1}}}$}
    \ENDFOR
    \STATE{$\hat{\psi}_T \gets \frac{1}{\sqrt{1 - \bar{\alpha}_T}} \sum_{i=1}^T \sqrt{\frac{\bar{\alpha}_T}{\bar{\alpha}_{i-1}}} \hat{\gamma}_i$}
    \FOR{$t = 1$ to $T$}
    \STATE{Normalize $\gamma_t \gets \hat{\gamma}_t / \hat{\psi}_T$}
    \ENDFOR
    \RETURN $\{\gamma_t\}_{t=1}^T$
  \end{algorithmic}
\end{algorithm}

\subsection{The conditional mean under the noise-matching strategy}
Under the noise-matching strategy for~$\gamma_t$, the following result holds:
\begin{proposition} \label{prop:nu}
  Suppose that~$\gamma_t$ is determined using the noise-matching strategy and~$\sigma_0 = 1$. Then, the conditional distribution~$q(\bm{x}_{t-1} | \bm{x}_{t}, \bm{x}_0, \bm{\xi})$ in \eqref{eq:q_x_t-1_x_0_xi} coincides with~$q(\bm{x}_{t-1} | \bm{x}_{t}, \bm{x}_0)$ in standard diffusion models~\cite{ho2020denoising}.
\end{proposition}
\begin{proof}
  From Appendix~\ref{apdx:proof_lemma_mu}, we have
  \[
    \phi_t = \psi_t - \frac{\sqrt{1 - \bar{\alpha}_t} \sqrt{\alpha_t}}{1 - \alpha_t} \nu_t \quad (t=2, \ldots, T).
  \]
  Under the noise-matching strategy,~$\phi_t = \psi_t$, which implies~$\nu_t = 0$. Substituting this into \eqref{eq:tilde_mu_x0},~$\tilde{\mu}_t(\bm{x}_t, \bm{x}_0, \bm{\xi})$ becomes independent of~$\bm{\xi}$. Therefore, the conditional distribution reduces to~$q(\bm{x}_{t-1} | \bm{x}_{t}, \bm{x}_0)$ when~$\sigma_0 = 1$, completing the proof.
\end{proof}

\subsection{Example calculation of the gamma coefficients}
We present a concrete example of computing~$\gamma_t$,~$\psi_t$, and~$\phi_t$ using the noise-matching strategy. As an illustration, we use the~$\beta_t$ schedule from Stable Diffusion 1.5~\cite{rombach2022high} with~$T = 1000$.
Figure~\ref{fig:plot_beta_gamma} shows the resulting~$\gamma_t$, together with the corresponding~$\phi_t$ and~$\psi_t$. The scale of~$\gamma_t$ is comparable to that of~$\beta_t$, but increases more rapidly at larger time steps. In addition,~$\phi_t$ and~$\psi_t$ coincide for all~$t$ and converge to~$1$ as~$t \to T$.
\begin{figure}[tbhp]
  \centering
  \includegraphics[width=1.0\linewidth]{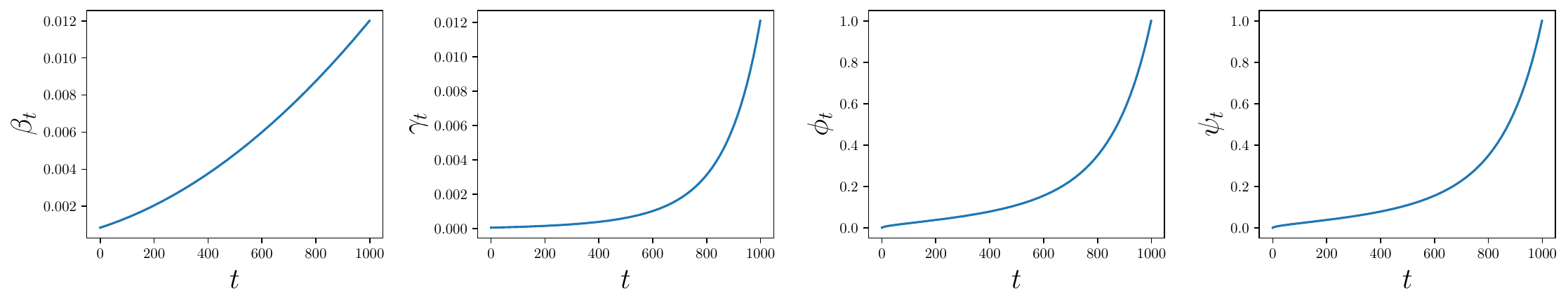}
  \caption{From left to right: $\beta_t$ from Stable Diffusion 1.5, and the corresponding $\gamma_t$, $\phi_t$, and $\psi_t$ computed using the noise-matching strategy.}
  \label{fig:plot_beta_gamma}
\end{figure}

As shown in Figure~\ref{fig:plot_beta_gamma}, both~$\phi_t$ and~$\psi_t$ increase with time~$t$. In the loss function~\eqref{eq:loss_proposed}, this implies that the contribution of the additional noise~$\bm{\xi}$ becomes larger at later time steps. Consequently, when~$\bm{x}_t$ is close to~$\bm{x}_0$, the coefficient applied to~$\bm{\xi}$ is small, preventing the additional noise from perturbing the data excessively in low-noise regimes. In contrast, at later time steps where $\bm{x}_t$ is dominated by noise, the influence of~$\bm{\xi}$ becomes more significant, making the effect of~$\bm{\xi}$ more prominent in high-noise regimes.
This behavior arises naturally from the condition~$\phi_t = \psi_t$ imposed by the noise-matching strategy.

\section{Extension to velocity prediction modeling} \label{sec:vpred}
The proposed model is grounded in a well-defined probabilistic framework, enabling principled integration with other diffusion modeling techniques, whereas such integrations are less straightforward in the offset noise model.
As a concrete example, we extend the proposed model to $v$-prediction~\cite{salimans2022progressive}, which is widely used in modern diffusion models, including recent text-to-image systems such as Stable Diffusion~2~\cite{rombach2022high,SD2}.
In this formulation,~$\mu_\theta$ is reparameterized using~$v_\theta$ (velocity) instead of~$\epsilon_\theta$.
Compared to $\epsilon$-prediction, $v$-prediction remains well-defined even when~$\alpha_t$ approaches zero, a regime where $\epsilon$-prediction becomes ill-conditioned due to \eqref{eq:mu_epsilon}.
This property has been exploited in~\cite{lin2024common} to address limitations of $\epsilon$-prediction in diffusion models.

\subsection{Training loss function in $v$-prediction modeling}
The following proposition defines the training loss function for the proposed model under $v$-prediction.
\begin{proposition}[Training loss function for $v$-prediction]\label{prop:loss_vpred}
  Suppose the forward and reverse processes are defined as in \eqref{eq:fp_proposed1}--\eqref{eq:rev_prop4}, and that~$\gamma_t \ (t=1,\ldots,T)$ is determined by the noise-matching strategy. Then, the objective that maximizes the ELBO in~\eqref{eq:elbo_prop} under $v$-prediction is
  \begin{align*}
    \ell^v(\theta;\bm{x}_0) = \mathbb{E}_{q(\bm{\xi}), \mathcal{U}(t | 1,T), \mathcal{N}(\bm{\epsilon}_0 | 0, I)}  \left[\lambda_t^v \left\|\sqrt{\bar{\alpha}_t}\left(\sigma_0 \bm{\epsilon}_0 + \psi_t \bm{\xi} \right) - \sqrt{1 - \bar{\alpha}_t} \bm{x}_0 - v_\theta(\bm{x}_t, t)\right\|^2 \right],
  \end{align*}
  where
  \begin{align*}
    \lambda_t^v = \frac{\bar{\alpha}_{t-1} (1-\alpha_t)^2}{2 \sigma_t^2 (1-\bar{\alpha}_t)},
  \end{align*}
  and~$\psi_t$ and~$\bm{x}_t$ are defined in \eqref{eq:def_psi} and \eqref{eq:xt_x0_ep0_f}, respectively.
\end{proposition}
\begin{proof}
  See Appendix~\ref{apdx:proof_prop_vpred}.
\end{proof}

\section{Related work}

This section situates the proposed model relative to prior studies on brightness-related failures of diffusion models and to broader approaches that relax the standard Gaussian terminal distribution.

\paragraph{Heuristic modifications to diffusion training}

Offset noise \cite{offset_noise} was introduced as an empirical technique for mitigating the difficulty that diffusion models have in generating images with extreme brightness levels. By adding an additional noise component correlated across channels, as described in Section~\ref{sec:offsetnoise}, offset noise has been shown empirically to improve the generation of low- and high-brightness images and has been adopted in practical systems \cite{podell2024sdxl}. A multi-scale extension of this idea, called pyramid noise, was proposed in \cite{pyramid_noise}.
Despite their empirical effectiveness, these methods directly modify the training objective without specifying corresponding forward and reverse processes. As a result, it remains unclear whether they are fully consistent with the likelihood-based formulation of diffusion models. In particular, the connection between these modified objectives and the underlying probabilistic framework is not made explicit, which limits their theoretical interpretability and their integration with other model variants.

\paragraph{Modifications of diffusion dynamics}

Another line of work addresses brightness-related issues by modifying the dynamics of the diffusion process. Lin et al. \cite{lin2024common} analyzed commonly used noise schedules and proposed adjusting the schedule so that the signal-to-noise ratio (SNR) approaches zero at the final time step. Although this approach improves the representation of low-frequency components, it introduces constraints under which the standard $\epsilon$-prediction formulation becomes inapplicable, thereby requiring alternative parameterizations such as $v$-prediction.
Hu et al. \cite{hu2024one} proposed a method that corrects the initial noise in the reverse process using an auxiliary model. Their approach can be applied to pre-trained diffusion models and improves the generation of low-frequency structures. However, it requires training an additional model and does not alter the underlying distributional assumptions of the diffusion process.
These approaches modify the forward or reverse dynamics to improve specific properties of generated samples, but they retain the fundamental assumption that the terminal distribution of the diffusion process is a zero-mean Gaussian.

\paragraph{Generalizing terminal distributions}

Beyond modifications to standard diffusion models, several studies have explored frameworks that relax the assumption that data must be diffused into a standard Gaussian distribution. Schr\"{o}dinger bridge methods \cite{de2021diffusion,liu20232,chenlikelihood} formulate generative modeling as the problem of learning stochastic processes that connect two arbitrary distributions. Similarly, flow-matching-based approaches \cite{lipmanflow,liuflow,tong2024improving} learn deterministic or stochastic flows between distributions without requiring the terminal distribution to be a standard Gaussian.
These approaches provide flexible frameworks for modeling transformations between distributions. In contrast, our method extends the discrete-time diffusion framework by allowing the terminal distribution to be Gaussian with an arbitrary mean structure while preserving the probabilistic formulation and variational training objective of standard diffusion models.

\section{Experiments} \label{sec:experiments}
In this section, we compare the proposed model with existing methods, focusing on the difficulty diffusion models have in generating images with extreme brightness levels.
Prior studies~\cite{offset_noise,lin2024common,hu2024one} have examined this issue in text-conditioned image generation by testing whether models can generate images such as truly black images from prompts like "Solid black background."
However, these evaluations were qualitative and focused on a narrow subset of the learned distribution, rather than providing a quantitative assessment of overall distribution modeling performance.

To the best of our knowledge, no benchmark image dataset currently provides both extreme brightness levels and a controlled underlying distribution.
To address this gap, we constructed synthetic data whose brightness distribution is uniform and used it to quantitatively evaluate the proposed method.
The experiments show that, especially in high-dimensional settings, existing diffusion models generate data with a non-uniform brightness distribution even when trained on data whose true brightness distribution is uniform.
In particular, samples with low or high brightness levels tend to be underrepresented.
These results indicate that the synthetic dataset used in this study exposes a concrete failure mode of conventional diffusion models.

We first describe the synthetic dataset and its statistical properties, and then present the experimental setup and results.

\subsection{Dataset} \label{sec:cylinder}
The synthetic dataset used in the experiments is referred to as the Cylinder dataset. It consists of data points~$\bm{x}_0 \in \mathbb{R}^n$ distributed in a cylindrical region of an $n$-dimensional space.
The centers of the top and bottom faces of the cylinder are defined as~$\bm{x}_\text{top} := k \mathbf{1}_n$ and~$\bm{x}_\text{bottom} := -\bm{x}_\text{top}$, respectively, where~$k \ (>0)$ is a scalar and~$\mathbf{1}_n$ is the~$n$-dimensional all-ones vector.
The radius of the cylinder is defined as~$r \|\mathbf{1}_n\| \ (r > 0)$.
Each data point~$\bm{x}_0$ is generated as
\begin{align}
  \bm{x}_0 = u_h \bm{x}_\text{top} + u_r \bm{x}_\text{ortho}, \label{eq:cylinder_x0}
\end{align}
where~$u_h$ and~$u_r$ are scalar random variables distributed as~$u_h \sim \mathcal{U}_c(-1, 1)$ and~$u_r \sim \mathcal{U}_c(0, r)$, respectively. Here,~$\mathcal{U}_c(a, b)$ denotes the uniform distribution over~$[a,b]$.
The vector~$\bm{x}_\text{ortho}$ is a random unit vector in the subspace~$\mathbf{1}_n^\perp$, which is orthogonal to~$\mathbf{1}_n$.
For reference, the Python code used to generate the Cylinder dataset is provided in Appendix~\ref{apdx:cylinder_code}.

\subsubsection{Brightness distribution of the Cylinder dataset} \label{sec:a0_dist}
Consider a grayscale image~$\bm{x}^\text{im}$ with~$n$ pixels. For convenience, we assume that each element of~$\bm{x}^\text{im}$ is normalized to lie in the range~$[-k, k]$.
Each element of~$\bm{x}^\text{im}$ represents the brightness of a pixel. The average brightness of~$\bm{x}^\text{im}$ is given by $B_n(\bm{x}^\text{im})$.
The image with the lowest average brightness is the one whose entries are all~$-k$ (a completely black image), whereas the image with the highest average brightness is the one whose entries are all~$k$ (a completely white image).

If the data points~$\bm{x}_0$ in the Cylinder dataset are interpreted as pseudo-grayscale images,\footnote{Strictly speaking,~$\bm{x}_0$ is not a true grayscale image because it does not necessarily lie in~$[-k, k]^n$.} then~$\bm{x}_\text{bottom}$ and~$\bm{x}_\text{top}$ correspond to a completely black image and a completely white image, respectively.
From \eqref{eq:cylinder_x0},~$\bm{x}_0$ can be viewed as the sum of two images,~$u_h \bm{x}_\text{top}$ and~$u_r \bm{x}_\text{ortho}$, whose average brightness values are
\begin{align*}
  B_n(u_h \bm{x}_\text{top})   & = \frac{1}{n} u_h \bm{x}_\text{top} \cdot \mathbf{1}_n = \frac{u_h k \mathbf{1}_n \cdot \mathbf{1}_n}{n} = u_h k \sim \mathcal{U}_c(-k, k), \\
  B_n(u_r \bm{x}_\text{ortho}) & = \frac{1}{n} u_r \bm{x}_\text{ortho} \cdot \mathbf{1}_n = 0,
\end{align*}
where we used the fact that~$\bm{x}_\text{ortho} \in \mathbf{1}_n^\perp$ implies~$\bm{x}_\text{ortho} \cdot \mathbf{1}_n = 0$.
Therefore, the average brightness of~$\bm{x}_0$ is
\begin{align}
  B_n(\bm{x}_0)
  &= \frac{1}{n} \left( u_h \bm{x}_\text{top} + u_r \bm{x}_\text{ortho} \right) \cdot \mathbf{1}_n \notag \\
  &= B_n(u_h \bm{x}_\text{top}) + B_n(u_r \bm{x}_\text{ortho}) \notag \\
  &= u_h k \sim \mathcal{U}_c(-k, k). \label{eq:avg_brightness}
\end{align}
Hence, if~$\bm{x}_0$ in the Cylinder dataset is interpreted as a pseudo-grayscale image, its average brightness is uniformly distributed over~$[-k, k]$.

\subsubsection{Experimental setup for the Cylinder dataset}
We varied the dimensionality as~$n=2, 10, 50, 100, 200$. For each value of~$n$, we generated training and test Cylinder datasets containing~$5000$ samples each by following the procedure described in Section~\ref{sec:cylinder}.
The parameters~$k$ and~$r$ were set to~$k=2$ and~$r=0.5$, respectively. These values were chosen so that the standard deviation of each component in the generated Cylinder dataset was close to~$1$.\footnote{The actual standard deviation of each component in the generated Cylinder dataset was approximately~$1.2$, independent of~$n$. In addition, by symmetry around the origin, the mean of each component was~$0$.}
An example of the Cylinder dataset with~$n=2$ is shown in the rightmost column of Figure~\ref{fig:diffusion_2d}.

\subsection{Compared models} \label{sec:model_compare}
We compared the following models:
\begin{itemize}
  \item \emph{Base model:} This model uses the training loss function~$\hat{\ell}_\text{simple}$ in~\eqref{eq:loss_eps_simple}, corresponding to the DDPM objective~\cite{ho2020denoising}.
  \item \emph{Offset noise model:} This model adopts the loss function~$\hat{\ell}_\text{offset}$ in \eqref{eq:loss_offsetnoise}. Since~$\bm{x}_0$ in the Cylinder dataset represents grayscale images (single-channel), we define~$q(\bm{\epsilon}_c) = \mathcal{N}(\bm{\epsilon}_c \ | \ 0, \sigma_c^2 \mathbf{1}_{n \times n})$.
  \item \emph{Zero-SNR model:} This model modifies~$\beta_t$ in the Base model using the method proposed in~\cite{lin2024common}.
  \item \emph{Proposed model:} This model uses the training loss function~$\ell_\text{simple}$ defined in~\eqref{eq:ell_simple}, where~$\gamma_t$ is determined by the noise-matching strategy and $\sigma_0=1$. In the proposed model,~$q(\bm{\xi})$ is set to be identical to~$q(\bm{\epsilon}_c)$ in the Offset noise model. Thus, in our experiments, the only difference between the proposed model and the offset noise model was the presence of the two time-dependent coefficients~$\phi_t$ and~$\psi_t$.
\end{itemize}
In addition, for each of the above models, we considered a version based on $v$-prediction \cite{salimans2022progressive}.
Although, as discussed in Section~\ref{sec:vpred}, there is no theoretical guarantee that offset noise remains valid under $v$-prediction, it can still be implemented in practice by replacing~$\bm{\epsilon}_0$ in the loss function with~$\bm{\epsilon}_0 + \bm{\epsilon}_c$, analogously to the $\epsilon$-prediction case. For the Zero-SNR model, only the $v$-prediction version was used because its formulation does not permit $\epsilon$-prediction.

For the Offset noise model, the hyperparameter~$\sigma_c^2$ was varied over 0.01, 0.05, 0.1, 0.5, and 1.0, and training and evaluation were conducted for each setting. Similarly, for the proposed model,~$\sigma_c^2$ was varied over 0.1, 0.5, and 1.0.

\subsection{Training and sampling settings}

\paragraph{Settings for the prediction target and noise schedule}
For~$\epsilon_\theta$ (or~$v_\theta$ in the $v$-prediction setting), we used a multilayer perceptron (MLP) with the time step~$t$ included as an additional input. The MLP had five hidden layers with GELU activations~\cite{gelu} and widths 256, 512, 1024, 512, and 256.
The maximum diffusion time was set to~$T = 200$, and~$\beta_t$ was determined using a log-linear schedule~\cite{permenterinterpreting}.\footnote{We used the \texttt{TimeInputMLP} and \texttt{ScheduleLogLinear} modules available at \url{https://github.com/yuanchenyang/smalldiffusion} for the MLP and beta schedule, respectively. In \texttt{ScheduleLogLinear}, we set \texttt{sigma\_min} to 0.01 and \texttt{sigma\_max} to 10.}

\paragraph{Optimizer settings}
We trained all models using the Adam optimizer \cite{adam} with learning rate~$0.001$. The mini-batch size was fixed at~$1024$, and training was run for~$200{,}000$ steps.
For some models, including the Base model, the loss occasionally diverged depending on the random seed. To mitigate this issue and stabilize training, we applied gradient clipping \cite{10.5555/3042817.3043083} with a maximum gradient norm of~$1$.

\paragraph{Settings for the reverse process}
When generating new data through the reverse process, we set the maximum time step to~$T = 200$. To prevent divergence, clipping was applied at each reverse step so that the samples remained within~$[-10, 10]^n$.\footnote{Such clipping is commonly used in image diffusion models. In this study, we chose the relatively large threshold~$10$, whereas the Cylinder dataset lies roughly in~$[-3, 3]^n$. This setting allows divergence to remain partially visible in the evaluation while avoiding numerical instability.}

\subsection{Evaluation metrics}
For each trained model, we generated~$5000$ samples through the reverse process and measured the distance between the generated distribution and the test-data distribution. We used two metrics: the 1-Wasserstein distance \cite{villani2008optimal} and the maximum mean discrepancy \cite{gretton2012kernel}, referred to below as 1WD and MMD, respectively.
For MMD, we used a Gaussian kernel with bandwidth~$\sqrt{n}$. We generated six train/test dataset pairs using different random seeds, and each model was trained and evaluated on all six pairs. Model initialization and other training factors were also randomized with the seed.

\subsection{Generation examples} \label{sec:generation_example}
Figure~\ref{fig:diffusion_2d} shows examples of data generated through the reverse process for~$n=2$. The top and bottom rows show the distributions at each time step for the Base model and the proposed model~($\sigma_c^2 = 1.0$), respectively. The rightmost column shows the test dataset.
For~$n=2$, both models produce samples at~$t=0$ whose distribution is close to that of the test data.
As described in Section~\ref{sec:model_compare}, the proposed model uses a terminal distribution whose mean is given by~$\bm{\xi} \sim \mathcal{N}\left(\bm{\xi} | 0, \sigma_c^2 \mathbf{1}_{n \times n} \right)$, whereas the Base model uses a zero-mean Gaussian at~$t=T$ ($T=200$ here). Consequently, at~$t=200$, the distribution of the proposed model is more spread along the diagonal directions than that of the Base model.
\begin{figure}[tbhp]
  \centering
  \includegraphics[width=1.0\linewidth,clip]{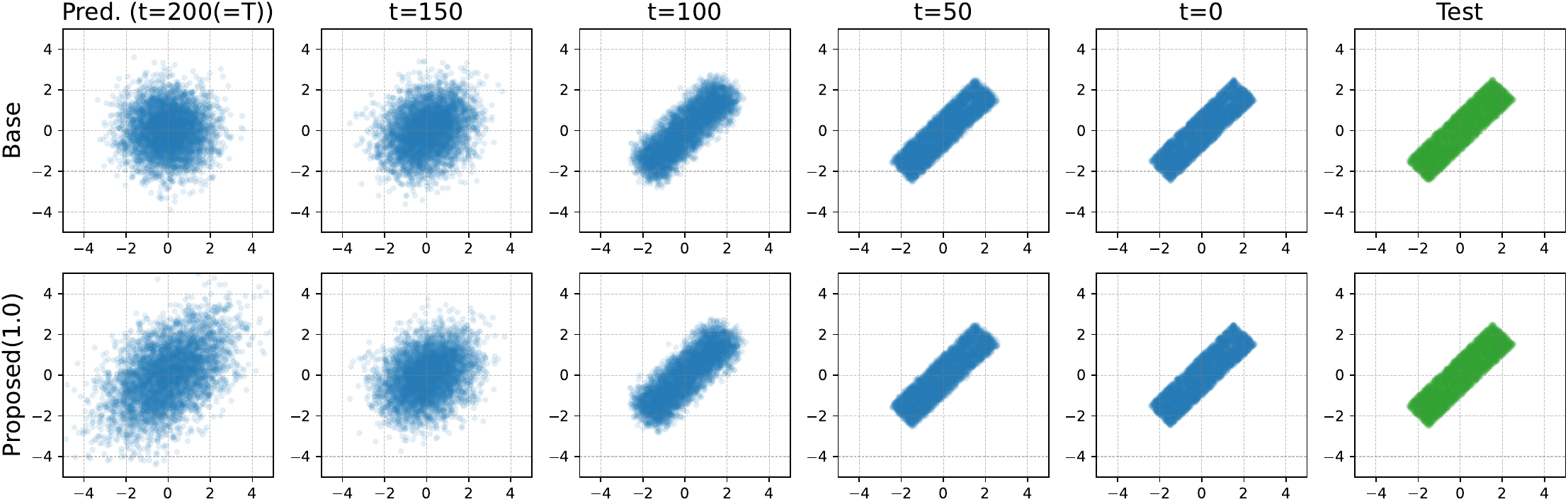}
  \caption{Distribution of generated data with~$n=2$ at each time step during the reverse process. The rightmost column represents the test data. The top row shows the results of the Base model, while the bottom row illustrates those of the Proposed model~($\sigma_c^2=1.0$).}
  \label{fig:diffusion_2d}
\end{figure}

\subsection{Evaluation results} \label{sec:eval_results}
\subsubsection{Comparison of average brightness distributions}
We compared the test dataset and the generated samples through the distribution of the average brightness~$B_n(\bm{x}_0)$.
As shown in \eqref{eq:avg_brightness}, the average brightness~$B_n(\bm{x}_0)$ in the Cylinder dataset follows the uniform distribution~$\mathcal{U}_c(-k, k)$, where~$k=2$ in our experiments.

The results are shown in Figure~\ref{fig:histgram}. For each~$n$, the top, middle, and bottom rows correspond to the Base model, the Offset noise model~($\sigma_c^2=0.1$), and the proposed model~($\sigma_c^2=1.0$), respectively. In each case, we use the model obtained after the final training step.
When~$n$ is small ($n \leq 10$), the distribution of~$B_n(\bm{x}_0)$ in the generated data closely matches that of the test dataset for all models.
As~$n$ increases, the~$B_n(\bm{x}_0)$ distribution generated by the Base and Offset noise models deviates from that of the test dataset. In particular, for the Base model with~$n=200$, samples near~$B_n(\bm{x}_0) \approx -2$ are underrepresented, highlighting the difficulty conventional diffusion models have in generating low-brightness images.
In contrast, the proposed model consistently produces samples whose~$B_n(\bm{x}_0)$ distribution remains close to that of the test dataset even as~$n$ increases.
This dimensional dependence is consistent with the theoretical analysis in Section~\ref{sec:brightness_theory}.
\begin{figure}[tbhp]
  \centering
  \includegraphics[width=1.0\linewidth,clip]{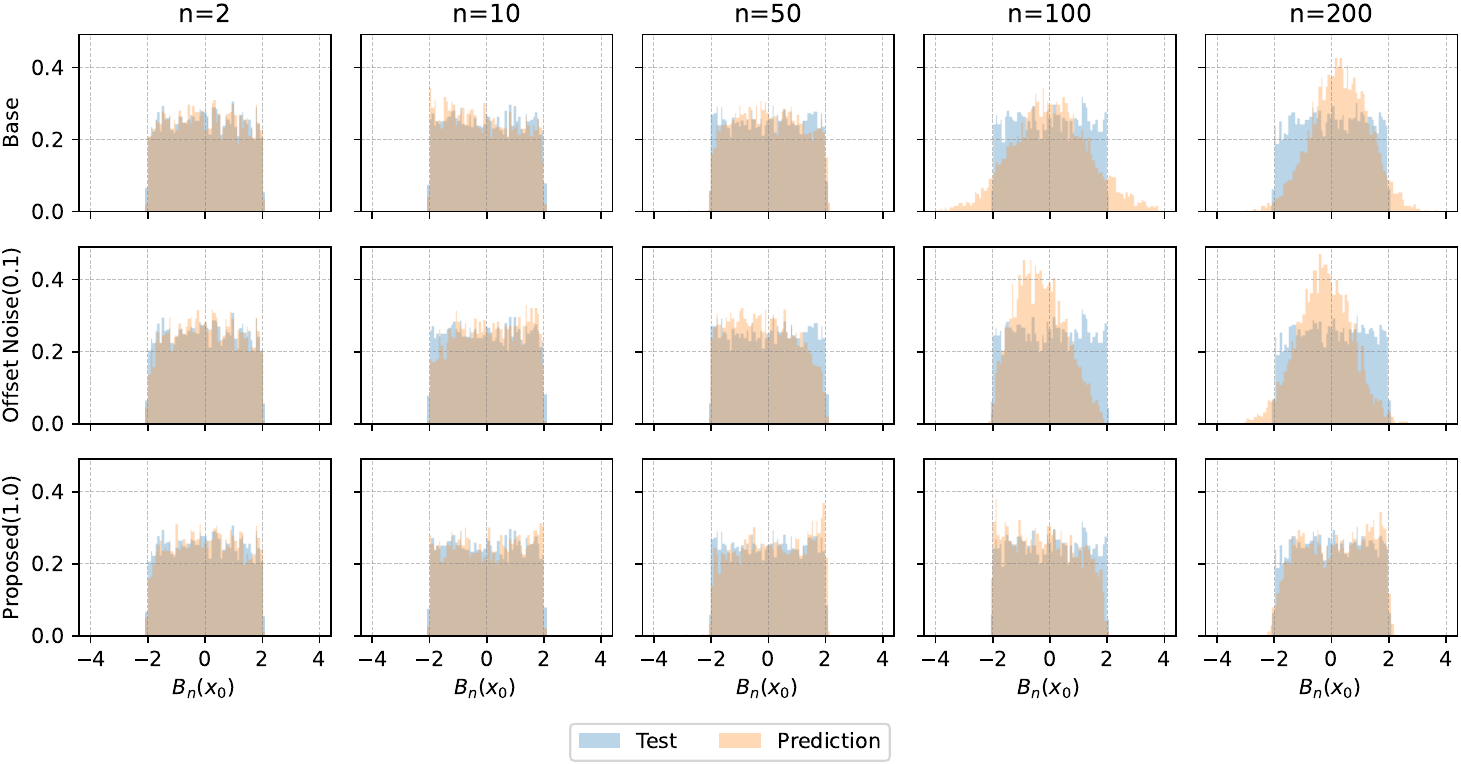}
  \caption{Comparison of distributions of average brightnesses~$B_n(\bm{x}_0)$ between the test data and the generated data.}
  \label{fig:histgram}
\end{figure}

\subsubsection{Comparison of quantitative metrics}
During training, every~$5000$ steps, we generated samples through the reverse process and measured their distance to the test dataset using 1WD and MMD. Figure~\ref{fig:1wd_mmd} reports the results for the $\epsilon$-prediction models. The curves show the median over six trials, and the error bars indicate the 10th to 90th percentiles.
For the Offset noise model, the results for~$\sigma_c^2=1.0$ were consistently worse than those for~$\sigma_c^2=0.5$, so the~$\sigma_c^2=1.0$ results are omitted for clarity.
\begin{figure}[tbhp]
  \centering
  \includegraphics[width=1.0\linewidth,clip]{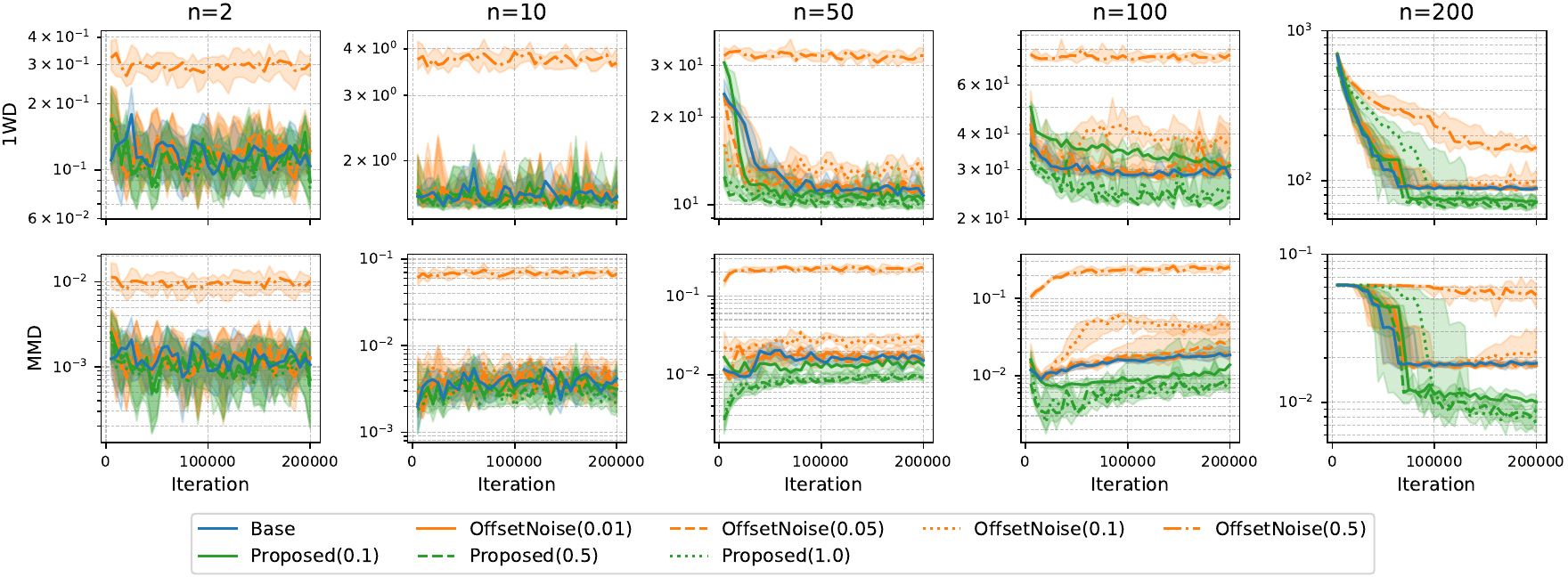}
  \caption{Evaluation results of 1WD (top row) and MMD (bottom row) during training.}
  \label{fig:1wd_mmd}
\end{figure}

Figure~\ref{fig:1wd_mmd} shows that for~$n \leq 10$, all models except the Offset noise model with~$\sigma_c^2=0.5$ achieve similar scores. As the dimensionality~$n$ increases, the proposed model outperforms the other methods by attaining smaller 1WD and MMD values.
These results suggest that the proposed model more accurately captures the distribution of the Cylinder dataset, especially in higher-dimensional settings.


\subsubsection{Training with data scaling} \label{sec:exp_scaling}
It is known that scaling the training data can affect the behavior of diffusion models~\cite{rombach2022high}. Instead of training directly on~$\bm{x}_0$, the diffusion model is trained on~$\bm{x}_0 / \rho$, where~$\rho\ (>0)$ is a scaling parameter.
After training, the final output is obtained by rescaling the generated data by~$\rho$.

The results for the Base model trained with data scaling on the Cylinder dataset are summarized in Appendix~\ref{apdx:exp_data_scaling}. For~$n=200$, data scaling does not substantially change the distribution of~$B_n(\bm{x}_0)$ in the generated samples. This suggests that data scaling alone does not resolve the difficulty of generating data with extreme average brightness.

\subsection{Evaluation results of $v$-prediction models}
Each model was also trained within the~$v$-prediction framework, and 1WD and MMD were evaluated every~$5000$ training steps. The results are shown in Figure~\ref{fig:1wd_mmd_vpred}.
\begin{figure}[tbhp]
  \begin{center}
    \includegraphics[width=1.0\linewidth,clip]{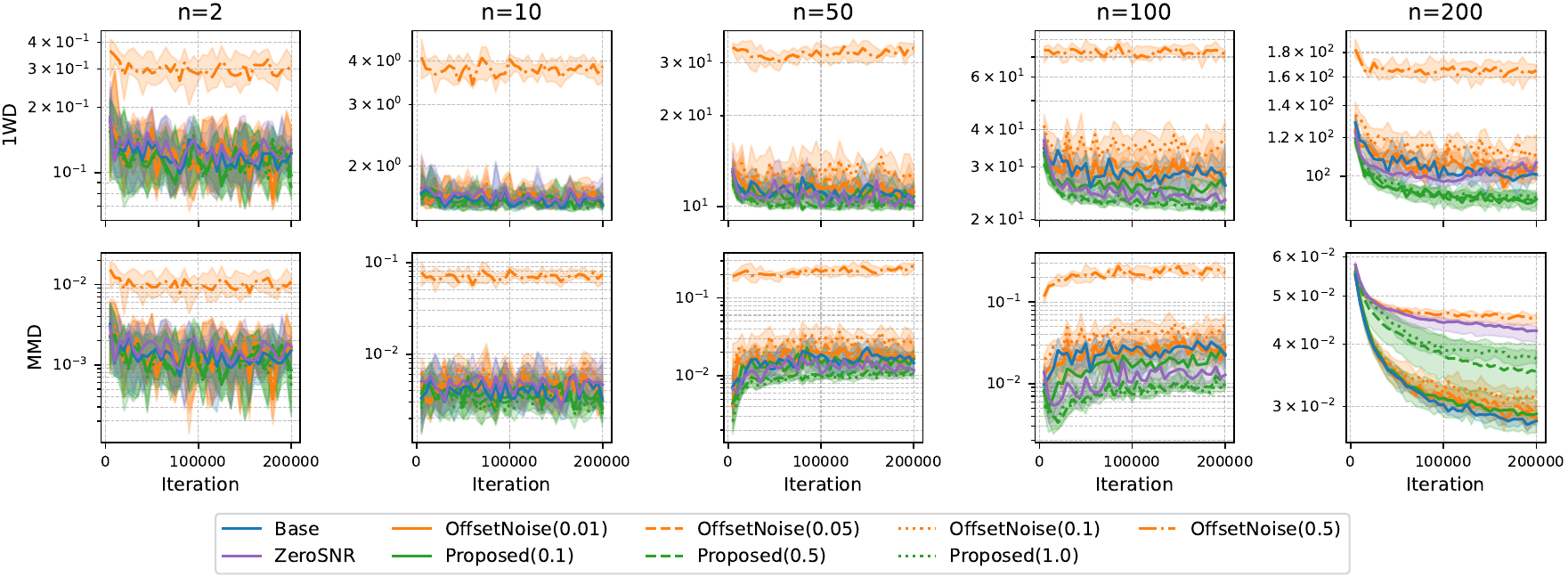}
    \caption{Evaluation results of 1WD (top) and MMD (bottom) during training within the~$v$-prediction framework.}
    \label{fig:1wd_mmd_vpred}
  \end{center}
\end{figure}

As in Figure~\ref{fig:1wd_mmd}, all models except the Offset noise model~($\sigma_c^2 = 0.5$) achieve comparable scores for~$n \leq 10$. As~$n$ increases, differences between the models become clearer. In particular, for~$n = 200$, the proposed model attains a lower 1WD than the other methods.
However, under MMD, the proposed model underperforms the Base model for~$n = 200$. A closer inspection revealed that, when sampling from the proposed model with~$n = 200$, a small number of points diverged during the reverse process and moved far from the test-data distribution. These outliers accounted for approximately~$10$ of the~$5000$ generated samples, or about~$0.2\%$ of the total. Because MMD is highly sensitive to outliers, these points likely degraded the MMD score.
In contrast, 1WD is less sensitive to such outliers. Therefore, the combination of higher MMD and lower 1WD in Figure~\ref{fig:1wd_mmd_vpred} suggests that, aside from a small number of divergent samples, the overall generated distribution is closer to the test distribution.

Appendix~\ref{apdx:brightness_vpred} compares the distributions of~$B_n(\bm{x}_0)$ for the test data and the samples generated by each $v$-prediction model. As in Figure~\ref{fig:histgram}, the distribution produced by the Base model departs further from the test distribution as~$n$ increases, whereas the proposed model remains closer to the test distribution even at~$n=200$.
These results suggest that the Base model still struggles to generate data with extreme brightness under $v$-prediction, whereas the proposed model substantially alleviates this difficulty.

\section{Conclusion and Future Work} \label{sec:conclusion}
We proposed a novel discrete-time diffusion model that introduces an additional random variable~$\bm{\xi} \sim q(\bm{\xi})$. We derived an ELBO for the proposed model and showed that the resulting loss function closely resembles the loss obtained by applying offset noise to conventional diffusion models.
This result provides a theoretical interpretation of offset noise, which has been empirically effective but has lacked a rigorous probabilistic foundation.
It also offers a broader perspective on offset noise and extends its applicability within a principled diffusion-modeling framework.

Several directions remain for future work. In this study, the distribution~$q(\bm{\xi})$ was predefined; an important extension would be to estimate~$q(\bm{\xi})$ in a data-driven manner. In addition, this paper considered the setting in which~$\bm{x}_0$ and~$\bm{\xi}$ are unpaired. Future work could investigate paired settings in which~$\bm{x}_0$ and~$\bm{\xi}$ are provided jointly. For example, one may consider a task in which~$\bm{x}_0$ is a high-resolution image and~$\bm{\xi}$ is the corresponding low-resolution image. Another important direction is to evaluate the proposed model on real-image datasets in order to assess whether the improvements observed on the synthetic benchmark translate to practical image-generation settings.

\appendix

\section{Proofs and formula derivations} \label{apdx:derivation}

\subsection{Derivation of the evidence lower bound} \label{apdx:elbo_prop_derivation}
\begin{align*}
  & \log p_\theta(\bm{x}_0)  \\
  & \geq \mathbb{E}_{q(\bm{x}_{1:T}, \bm{\xi} | \bm{x}_0)}\left[\log \frac{p_\theta(\bm{x}_{0:T}, \bm{\xi})}{q(\bm{x}_{1:T}, \bm{\xi} | \bm{x}_0)}\right]   \\
  & = \mathbb{E}_{q(\bm{x}_{1:T}, \bm{\xi} | \bm{x}_0)}\left[\log \frac{\cancel{p(\bm{\xi})} p(\bm{x}_T | \bm{\xi}) \prod_{t=1}^T p_\theta(\bm{x}_{t-1}|\bm{x}_t)}{\cancel{q(\bm{\xi})} \prod_{t=1}^T q(\bm{x}_{t} | \bm{x}_{t-1}, \bm{\xi})}\right]  \\
  & = \mathbb{E}_{q(\bm{x}_{1:T}, \bm{\xi} | \bm{x}_0)}\left[\log \frac{p(\bm{x}_T | \bm{\xi}) p_\theta(\bm{x}_0 | \bm{x}_1) \prod_{t=2}^T p_\theta(\bm{x}_{t-1}|\bm{x}_t)}{q(\bm{x}_1 | \bm{x}_0, \bm{\xi}) \prod_{t=2}^T q(\bm{x}_{t} | \bm{x}_{t-1}, \bm{\xi})}\right]                \\
  & = \mathbb{E}_{q(\bm{x}_{1:T}, \bm{\xi} | \bm{x}_0)}\left[\log \frac{p(\bm{x}_T | \bm{\xi}) p_\theta(\bm{x}_0 | \bm{x}_1) \prod_{t=2}^T p_\theta(\bm{x}_{t-1}|\bm{x}_t)}{q(\bm{x}_1 | \bm{x}_0, \bm{\xi}) \prod_{t=2}^T q(\bm{x}_{t} | \bm{x}_{t-1}, \bm{x}_0, \bm{\xi})}\right]      \\
  & = \mathbb{E}_{q(\bm{x}_{1:T}, \bm{\xi} | \bm{x}_0)}\left[\log \frac{ p(\bm{x}_T | \bm{\xi}) p_\theta(\bm{x}_0 | \bm{x}_1) }{q(\bm{x}_1 | \bm{x}_0, \bm{\xi})} + \log \prod_{t=2}^T  \frac{p_\theta(\bm{x}_{t-1}|\bm{x}_t)}{q(\bm{x}_{t} | \bm{x}_{t-1}, \bm{x}_0, \bm{\xi})}\right]  \\
  & = \mathbb{E}_{q(\bm{x}_{1:T}, \bm{\xi} | \bm{x}_0)}\left[\log \frac{p(\bm{x}_T | \bm{\xi}) p_\theta(\bm{x}_0 | \bm{x}_1) }{q(\bm{x}_1 | \bm{x}_0, \bm{\xi}) } + \log \prod_{t=2}^T  \frac{p_\theta(\bm{x}_{t-1}|\bm{x}_t)}{\frac{q(\bm{x}_{t-1} | \bm{x}_{t}, \bm{x}_0, \bm{\xi}) \cancel{q(\bm{x}_t | \bm{x}_0, \bm{\xi})}}{\cancel{q(\bm{x}_{t-1}|\bm{x}_0,\bm{\xi})}}}\right]  \\
  & = \mathbb{E}_{q(\bm{x}_{1:T}, \bm{\xi} | \bm{x}_0)}\left[\log \frac{p(\bm{x}_T | \bm{\xi}) p_\theta(\bm{x}_0 | \bm{x}_1) }{\cancel{q(\bm{x}_1 | \bm{x}_0, \bm{\xi})}} + \log \frac{\cancel{q(\bm{x}_1|\bm{x}_0, \bm{\xi})}}{q(\bm{x}_T|\bm{x}_0, \bm{\xi})} + \log \prod_{t=2}^T  \frac{p_\theta(\bm{x}_{t-1}|\bm{x}_t)}{q(\bm{x}_{t-1} | \bm{x}_{t}, \bm{x}_0, \bm{\xi})}\right] \\
  & = \mathbb{E}_{q(\bm{x}_{1:T}, \bm{\xi} | \bm{x}_0)}\left[\log \frac{p(\bm{x}_T | \bm{\xi}) p_\theta(\bm{x}_0 | \bm{x}_1) }{q(\bm{x}_T|\bm{x}_0, \bm{\xi})} + \sum_{t=2}^T \log \frac{p_\theta(\bm{x}_{t-1}|\bm{x}_t)}{q(\bm{x}_{t-1} | \bm{x}_{t}, \bm{x}_0, \bm{\xi})}\right]  \\
  & = \mathbb{E}_{q(\bm{x}_{1:T}, \bm{\xi} | \bm{x}_0)}\left[\log p_\theta(\bm{x}_0 | \bm{x}_1)\right] + \mathbb{E}_{q(\bm{x}_{1:T}, \bm{\xi} | \bm{x}_0)}\left[\log \frac{p(\bm{x}_T | \bm{\xi})}{q(\bm{x}_T|\bm{x}_0, \bm{\xi})}\right] \\
  & \qquad + \sum_{t=2}^T \mathbb{E}_{q(\bm{x}_{1:T}, \bm{\xi} | \bm{x}_0)}\left[\log \frac{p_\theta(\bm{x}_{t-1}|\bm{x}_t)}{q(\bm{x}_{t-1} | \bm{x}_{t}, \bm{x}_0, \bm{\xi})}\right] \\
  & = \mathbb{E}_{q(\bm{x}_{1}, \bm{\xi} | \bm{x}_0)}\left[\log p_\theta(\bm{x}_0 | \bm{x}_1)\right] + \mathbb{E}_{q(\bm{x}_{T}, \bm{\xi} | \bm{x}_0)}\left[\log \frac{p(\bm{x}_T | \bm{\xi})}{q(\bm{x}_T|\bm{x}_0, \bm{\xi})}\right] \\
  & \qquad + \sum_{t=2}^T \mathbb{E}_{q(\bm{x}_{t}, \bm{x}_{t-1}, \bm{\xi} | \bm{x}_0)}\left[\log \frac{p_\theta(\bm{x}_{t-1}|\bm{x}_t)}{q(\bm{x}_{t-1} | \bm{x}_{t}, \bm{x}_0, \bm{\xi})}\right] \\
  & = \mathbb{E}_{q(\bm{\xi})} \mathbb{E}_{q(\bm{x}_{1} | \bm{x}_0, \bm{\xi})}\left[\log p_\theta(\bm{x}_0 | \bm{x}_1)\right] + \mathbb{E}_{q(\bm{\xi})} \mathbb{E}_{q(\bm{x}_{T} | \bm{x}_0, \bm{\xi})}\left[\log \frac{p(\bm{x}_T | \bm{\xi})}{q(\bm{x}_T | \bm{x}_0, \bm{\xi})}\right]          \\
  & \qquad + \sum_{t=2}^T \mathbb{E}_{q(\bm{\xi})} \mathbb{E}_{q(\bm{x}_{t} | \bm{x}_0, \bm{\xi})} \mathbb{E}_{q(\bm{x}_{t-1} | \bm{x}_{t}, \bm{x}_0, \bm{\xi})}\left[\log \frac{p_\theta(\bm{x}_{t-1}|\bm{x}_t)}{q(\bm{x}_{t-1} | \bm{x}_{t}, \bm{x}_0, \bm{\xi})}\right] \notag \\
  & = \mathbb{E}_{q(\bm{\xi})} \mathbb{E}_{q(\bm{x}_{1} | \bm{x}_0, \bm{\xi})}\left[\log p_\theta(\bm{x}_0 | \bm{x}_1)\right] - \mathbb{E}_{q(\bm{\xi})} \left[ D_{\text{KL}} \left(q(\bm{x}_T | \bm{x}_0, \bm{\xi}) \ || \ p(\bm{x}_T | \bm{\xi})\right) \right]  \\
  & \qquad - \sum_{t=2}^T \mathbb{E}_{q(\bm{\xi})} \mathbb{E}_{q(\bm{x}_{t} | \bm{x}_0, \bm{\xi})} \left[ D_{\text{KL}}\left(q(\bm{x}_{t-1} | \bm{x}_{t}, \bm{x}_0, \bm{\xi}) \ || \ p_\theta(\bm{x}_{t-1}|\bm{x}_t) \right)\right]. \notag
\end{align*}

\subsection{Derivation of the expression for the latent variable} \label{apdx:xt_x0_ep0_f_derivation}
Suppose that we have $2T$ random variables $\left\{\bm{\epsilon}_t^\ast, \bm{\epsilon}_t\right\}_{t=0}^{T-1} \overset{\text{iid}}{\sim} \mathcal{N}(\bm{\epsilon} \ | \ 0,I)$ and $\bm{\xi} \sim q(\bm{\xi})$. Then, for $t=1,\ldots,T$, we have
\begin{align}
  \bm{x}_t & = \sqrt{1 - \beta_t}\left( \bm{x}_{t-1} + \gamma_t \bm{\xi} \right) +\sqrt{\beta_t} \sigma_0 \bm{\epsilon}_{t-1}  \notag \\
  & = \sqrt{\alpha_t} \bm{x}_{t-1} + \sqrt{1-\alpha_t} \sigma_0 \bm{\epsilon}_{t-1} + \sqrt{\alpha_t} \gamma_t \bm{\xi}  \notag \\
  & = \sqrt{\alpha_t} \left(\sqrt{\alpha_{t-1}} \bm{x}_{t-2} +\sqrt{1-\alpha_{t-1}} \sigma_0 \bm{\epsilon}_{t-2}^\ast + \sqrt{\alpha_{t-1}} \gamma_{t-1} \bm{\xi}\right) + \sqrt{1-\alpha_t} \sigma_0 \bm{\epsilon}_{t-1} +\sqrt{\alpha_t} \gamma_t \bm{\xi} \notag \\
  & = \sqrt{\alpha_t \alpha_{t-1}} \bm{x}_{t-2} + \sqrt{\alpha_t - \alpha_t\alpha_{t-1}} \sigma_0 \bm{\epsilon}_{t-2}^\ast + \sqrt{1-\alpha_t} \sigma_0 \bm{\epsilon}_{t-1} + \left(\sqrt{\alpha_t} \gamma_t + \sqrt{\alpha_t \alpha_{t-1}} \gamma_{t-1}\right) \bm{\xi} \label{eq:marge_epsilon1}    \\
  & = \sqrt{\alpha_t \alpha_{t-1}} \bm{x}_{t-2} + \sqrt{1 - \alpha_t \alpha_{t-1}} \sigma_0 \bm{\epsilon}_{t-2}  + \left(\sqrt{\alpha_t} \gamma_t + \sqrt{\alpha_t \alpha_{t-1}} \gamma_{t-1}\right) \bm{\xi}  \label{eq:marge_epsilon2}  \\
  & = \ldots  \notag  \\
  & = \sqrt{\prod_{i=1}^t \alpha_i} \bm{x}_0 + \sqrt{1 - \prod_{i=1}^t \alpha_i} \sigma_0 \bm{\epsilon}_0 + \sum_{i=1}^t \sqrt{\prod_{j=i}^t \alpha_j} \gamma_i \bm{\xi}  \notag  \\
  & = \sqrt{\prod_{i=1}^t \alpha_i} \bm{x}_0 + \sqrt{1 - \prod_{i=1}^t \alpha_i} \sigma_0 \bm{\epsilon}_0 + \sum_{i=1}^t \sqrt{\frac{\prod_{j=0}^t \alpha_j}{\prod_{j=0}^{i-1} \alpha_j}} \gamma_i \bm{\xi}  \notag \\
  & = \sqrt{\bar{\alpha}_t} \bm{x}_0 + \sqrt{1 - \bar{\alpha}_t} \sigma_0 \bm{\epsilon}_0 + \sum_{i=1}^t \sqrt{\frac{\bar{\alpha}_t}{\bar{\alpha}_{i-1}}} \gamma_i \bm{\xi} \notag  \\
  & = \sqrt{\bar{\alpha}_t} \bm{x}_0 + \sqrt{1 - \bar{\alpha}_t} \left(\sigma_0 \bm{\epsilon}_0 + \frac{1}{\sqrt{1 - \bar{\alpha}_t}} \sum_{i=1}^t \sqrt{\frac{\bar{\alpha}_t}{\bar{\alpha}_{i-1}}} \gamma_i \bm{\xi}\right)  \notag  \\
  & = \sqrt{\bar{\alpha}_t} \bm{x}_0 + \sqrt{1 - \bar{\alpha}_t} \left(\sigma_0 \bm{\epsilon}_0 + \psi_t \bm{\xi}\right). \notag
\end{align}
See~\cite{luo2022understanding} for the transformation from \eqref{eq:marge_epsilon1} to \eqref{eq:marge_epsilon2}.

\subsection{Derivation of the conditional Gaussian expressions} \label{apdx:q_xt_x0_xi_derivation}
For $t=2,\ldots,T$, we have
\begin{align*}
  & q(\bm{x}_{t-1} | \bm{x}_{t}, \bm{x}_0, \bm{\xi})                                                                                                                                                                                                                                                                                                      \\
  & = \frac{q(\bm{x}_t | \bm{x}_{t-1}, \bm{x}_0, \bm{\xi})q(\bm{x}_{t-1} | \bm{x}_0, \bm{\xi})}{q(\bm{x}_t | \bm{x}_0, \bm{\xi})}  \\
  & \propto q(\bm{x}_t | \bm{x}_{t-1}, \bm{x}_0, \bm{\xi})q(\bm{x}_{t-1} | \bm{x}_0, \bm{\xi})  \\
& = \mathcal{N}\left(\bm{x}_t \ \middle| \ \sqrt{\alpha_t}(\bm{x}_{t-1} + \gamma_t \bm{\xi}), (1 - \alpha_t) \sigma_0^2 I)\right) \mathcal{N}\left(\bm{x}_{t-1} \ \middle| \ \sqrt{\bar{\alpha}_{t-1}} \bm{x}_0  + \sqrt{1 - \bar{\alpha}_{t-1}} \psi_{t-1} \bm{\xi}, \left(1 - \bar{\alpha}_{t-1}\right) \sigma_0^2 I\right)  \\
& = \mathcal{N}\left(\bm{x}_{t-1} \ \middle| \ \frac{1}{\sqrt{\alpha_t}} \bm{x}_{t} - \gamma_t \bm{\xi}, \frac{1 - \alpha_t}{\alpha_t} \sigma_0^2 I)\right) \mathcal{N}\left(\bm{x}_{t-1} \ \middle| \ \sqrt{\bar{\alpha}_{t-1}} \bm{x}_0  + \sqrt{1 - \bar{\alpha}_{t-1}} \psi_{t-1} \bm{\xi}, \left(1 - \bar{\alpha}_{t-1}\right) \sigma_0^2 I\right) \\
& \propto \mathcal{N}\left(\bm{x}_{t-1} \ \middle| \ \tilde{\mu}(\bm{x}_t, \bm{x}_0, \bm{\xi}), \tilde{\beta}_t I \right),
\end{align*}
where $\tilde{\mu}(\bm{x}_t, \bm{x}_0, \bm{\xi})$ and $\tilde{\beta}_t$ are obtained by multiplying the two normal distributions:
\begin{align*}
\tilde{\mu}(\bm{x}_t, \bm{x}_0, \bm{\xi}) & = \frac{1}{\frac{1 - \alpha_t}{\alpha_t} + (1 - \bar{\alpha}_{t-1})} \left((1-\bar{\alpha}_{t-1})\left(\frac{1}{\sqrt{\alpha}_t}\bm{x}_t - \gamma_t \bm{\xi}\right) + \frac{1 - \alpha_t}{\alpha_t}\left(\sqrt{\bar{\alpha}_{t-1}} \bm{x}_0  + \sqrt{1 - \bar{\alpha}_{t-1}} \psi_{t-1} \bm{\xi}\right)\right) \\
& = \frac{\alpha_t}{1 - \bar{\alpha}_t}\left(\frac{1-\bar{\alpha}_{t-1}}{\sqrt{\alpha_t}} \bm{x}_t - (1-\bar{\alpha}_{t-1}) \gamma_t \bm{\xi} + \frac{(1-\alpha_t)\sqrt{\bar{\alpha}_{t-1}}}{\alpha_t} \bm{x}_0 + \frac{(1-\alpha_t)\sqrt{1 - \bar{\alpha}_{t-1}} \psi_{t-1}}{\alpha_t}  \bm{\xi} \right)        \\
& = \frac{\sqrt{\alpha_t} (1-\bar{\alpha}_{t-1})}{1 - \bar{\alpha}_t} \bm{x}_t + \frac{(1 - \alpha_t)\sqrt{\bar{\alpha}_{t-1}}}{1 - \bar{\alpha}_t} \bm{x}_0 + \nu_t \bm{\xi},                                                                                                                                   \\
\tilde{\beta}_t                           & = \frac{\frac{1-\alpha_t}{\alpha_t} \sigma_0^2 \left(1-\bar{\alpha}_{t-1}\right) \sigma_0^2}{\frac{1-\alpha_t}{\alpha_t} \sigma_0^2 + \left(1-\bar{\alpha}_{t-1}\right) \sigma_0^2}
= \frac{(1 - \alpha_t)(1 - \bar{\alpha}_{t-1})}{1 - \bar{\alpha}_t} \sigma_0^2.
\end{align*}

\subsection{Proof of the lemma on the conditional mean} \label{apdx:proof_lemma_mu}
From \eqref{eq:xt_x0_ep0_f}, we have
\begin{align*}
\bm{x}_0 = \frac{1}{\sqrt{\bar{\alpha}_t}}\left(\bm{x}_t - \sqrt{1 - \bar{\alpha}_t}\left(\sigma_0 \bm{\epsilon}_0 + \psi_t \bm{\xi}\right)\right).
\end{align*}
Substituting this into \eqref{eq:tilde_mu_x0} yields
\begin{align*}
\tilde{\mu}(\bm{x}_t, \bm{x}_0, \bm{\xi}) & = \frac{\sqrt{\alpha_t} (1-\bar{\alpha}_{t-1})}{1 - \bar{\alpha}_t} \bm{x}_t + \frac{(1 - \alpha_t)\sqrt{\bar{\alpha}_{t-1}}}{1 - \bar{\alpha}_t} \frac{1}{\sqrt{\bar{\alpha}_t}}\left(\bm{x}_t - \sqrt{1 - \bar{\alpha}_t}\left(\sigma_0 \bm{\epsilon}_0 + \psi_t \bm{\xi}\right)\right) + \nu_t \bm{\xi} \\
& = \frac{1}{\sqrt{\alpha_t}} \bm{x}_t - \frac{1 - \alpha_t}{\sqrt{1 - \bar{\alpha}_t} \sqrt{\alpha_t}} \left(\sigma_0 \bm{\epsilon}_0 + \left(\psi_t  - \frac{\sqrt{1 - \bar{\alpha}_t} \sqrt{\alpha_t}}{1 - \alpha_t} \nu_t \right) \bm{\xi}\right)                                                        \\
& = \frac{1}{\sqrt{\alpha_t}} \bm{x}_t - \frac{1 - \alpha_t}{\sqrt{1 - \bar{\alpha}_t} \sqrt{\alpha_t}} \left(\sigma_0 \bm{\epsilon}_0 + \phi_t \bm{\xi}\right),
\end{align*}
where $\phi_t = \psi_t  - \frac{\sqrt{1 - \bar{\alpha}_t} \sqrt{\alpha_t}}{1 - \alpha_t} \nu_t$.
We can then expand $\phi_t$ as follows:
\begin{align}
\phi_t & = \psi_t  - \frac{\sqrt{1 - \bar{\alpha}_t} \sqrt{\alpha_t}}{1 - \alpha_t} \nu_t  \notag  \\
& =
\psi_t  - \frac{\sqrt{1 - \bar{\alpha}_t} \sqrt{\alpha_t}}{1 - \alpha_t} \frac{(1 - \alpha_t)\sqrt{1 - \bar{\alpha}_{t-1}} \psi_{t-1} - \alpha_t (1-\bar{\alpha}_{t-1})\gamma_t}{1 - \bar{\alpha}_t}  \notag \\
& = \psi_t - \frac{\sqrt{1- \bar{\alpha}_{t-1}}\sqrt{\alpha_t}}{\sqrt{1-\bar{\alpha_t}}} \psi_{t-1} + \frac{\alpha_t \sqrt{\alpha_t} (1-\bar{\alpha}_{t-1})}{(1-\alpha_t)\sqrt{1 - \bar{\alpha}_t}} \gamma_t. \label{eq:phi1}
\end{align}
From the definition of~$\psi_t$, we obtain
\begin{align}
\psi_t = \frac{\sqrt{\alpha_t}\sqrt{1 - \bar{\alpha}_{t-1}}}{\sqrt{1 - \bar{\alpha}_t}} \psi_{t-1} + \frac{\sqrt{\alpha_t}}{\sqrt{1 - \bar{\alpha}_t}} \gamma_t,  \label{eq:psi_t_psi_t-1}
\end{align}
Substituting \eqref{eq:psi_t_psi_t-1} into \eqref{eq:phi1}, we obtain
\begin{align*}
\phi_t & = \frac{\sqrt{\alpha_t}\sqrt{1 - \bar{\alpha}_{t-1}}}{\sqrt{1 - \bar{\alpha}_t}} \psi_{t-1} + \frac{\sqrt{\alpha_t}}{\sqrt{1 - \bar{\alpha}_t}} \gamma_t - \frac{\sqrt{1- \bar{\alpha}_{t-1}}\sqrt{\alpha_t}}{\sqrt{1-\bar{\alpha_t}}} \psi_{t-1} + \frac{\alpha_t \sqrt{\alpha_t} (1-\bar{\alpha}_{t-1})}{(1-\alpha_t)\sqrt{1 - \bar{\alpha}_t}} \gamma_t \\
& = \frac{\sqrt{\alpha_t}(1 - \alpha_t) + \alpha_t \sqrt{\alpha_t} (1-\bar{\alpha}_{t-1})}{(1-\alpha_t)\sqrt{1 - \bar{\alpha}_t}}  \\
& = \frac{\sqrt{\alpha_t} \sqrt{1 - \bar{\alpha}_t}}{1-\alpha_t} \gamma_t.
\end{align*}

\subsection{Derivation of the $\mathcal{L}_1$ term} \label{apdx:L1_derivation}
For the $\mathcal{L}_1$ term, from \eqref{eq:xt_x0_ep0_f} with $t=1$, we have
\begin{align*}
& \mathbb{E}_{q(\bm{x}_{1} | \bm{x}_0, \bm{\xi})}\left[\log p_\theta(\bm{x}_0 | \bm{x}_1) \right]  \\
& = \mathbb{E}_{q(\bm{x}_{1} | \bm{x}_0, \bm{\xi})}\left[ \log \mathcal{N}\left(\bm{x}_0 \ \middle| \ \mu_\theta(\bm{x}_1, 1), \sigma_1^2 I\right) \right]  \\
& = - \frac{1}{2 \sigma_1^2} \mathbb{E}_{q(\bm{x}_{1} | \bm{x}_0, \bm{\xi})}\left[\left\|\bm{x}_0 - \mu_\theta(\bm{x}_1, 1)\right\|^2 \right] + C_2  \\
& = - \frac{1}{2 \sigma_1^2} \mathbb{E}_{q(\bm{x}_{1} | \bm{x}_0, \bm{\xi})}\left[ \left\|\bm{x}_0 - \left(\frac{1}{\sqrt{\alpha_1}} \bm{x}_1 - \frac{1 - \alpha_1}{\sqrt{1 - \bar{\alpha}_1} \sqrt{\alpha_1}} \epsilon_\theta (\bm{x}_1, 1) \right)\right\|^2 \right] + C_2  \\
& = - \frac{1}{2 \sigma_1^2} \mathbb{E}_{\mathcal{N}(\bm{\epsilon}_0 | 0, I)} \left[ \left\|\bm{x}_0 - \left(\frac{1}{\sqrt{\alpha_1}} \left(\sqrt{\alpha}_1 \bm{x}_0 + \sqrt{1 - \alpha_1} \left(\sigma_0 \bm{\epsilon}_0 + \psi_1 \bm{\xi}\right)\right) - \frac{\sqrt{1 - \alpha_1}}{\sqrt{\alpha_1}} \epsilon_\theta (\bm{x}_1, 1) \right)\right\|^2 \right] + C_2 \\
& = - \frac{1}{2 \sigma_1^2} \mathbb{E}_{\mathcal{N}(\bm{\epsilon}_0 | 0, I)} \left[\left\|\frac{\sqrt{1-\alpha_1}}{\sqrt{\alpha_1}}\left(\sigma_0 \bm{\epsilon}_0 + \psi_1 \bm{\xi} - \epsilon_\theta (\bm{x}_1, 1)\right) \right\|^2 \right] + C_2  \\
& = - \frac{1 - \alpha_1}{2 \sigma_1^2 \alpha_1} \mathbb{E}_{\mathcal{N}(\bm{\epsilon}_0 | 0, I)} \left[ \left\|\sigma_0 \bm{\epsilon}_0 + \psi_1 \bm{\xi} - \epsilon_\theta (\bm{x}_1, 1) \right\|^2 \right] + C_2  \\
& = - \frac{1 - \alpha_1}{2 \sigma_1^2 \alpha_1} \mathbb{E}_{\mathcal{N}(\bm{\epsilon}_0 | 0, I)} \left[ \left\|\sigma_0 \bm{\epsilon}_0 + \phi_1 \bm{\xi} - \epsilon_\theta (\bm{x}_1, 1) \right\|^2 \right] + C_2  \\
& = - \lambda_1 \mathbb{E}_{\mathcal{N}(\bm{\epsilon}_0 | 0, I)} \left[ \left\|\sigma_0 \bm{\epsilon}_0 + \phi_1 \bm{\xi} - \epsilon_\theta (\bm{x}_1, 1) \right\|^2 \right] + C_2.
\end{align*}

\subsection{Proof of Proposition~\ref{prop:brightness_stat}} \label{apdx:proof_brightness_stat}
Because~$q(\bm{\xi}) = \mathcal{N}(\bm{\xi} \mid 0, \sigma_c^2 \mathbf{1}_{n \times n})$ is a rank-one Gaussian supported on~$\mathrm{span}\{\mathbf{1}_n\}$, there exists a scalar Gaussian random variable~$z \sim \mathcal{N}(0,\sigma_c^2)$ such that
\begin{align}
\bm{\xi} = z \mathbf{1}_n \quad \text{a.s.} \label{eq:xi_rank_one}
\end{align}
Applying the linear functional~$B_n(\bm{x}) = n^{-1}\mathbf{1}_n^\top \bm{x}$ to \eqref{eq:xt_x0_ep0_f}, we obtain
\begin{align*}
B_n(\bm{x}_t)
&= \sqrt{\bar{\alpha}_t} B_n(\bm{x}_0) + \sqrt{1-\bar{\alpha}_t}\left(\sigma_0 B_n(\bm{\epsilon}_0) + \psi_t B_n(\bm{\xi})\right).
\end{align*}
By \eqref{eq:xi_rank_one},~$B_n(\bm{\xi}) = z$. In addition,
\begin{align*}
B_n(\bm{\epsilon}_0) = \frac{1}{n}\sum_{i=1}^n \epsilon_{0,i} \sim \mathcal{N}\left(0,\frac{1}{n}\right),
\end{align*}
because the entries of~$\bm{\epsilon}_0$ are independent standard normal variables. Denoting~$\varepsilon_B := B_n(\bm{\epsilon}_0)$ proves \eqref{eq:brightness_prop}. Since~$\varepsilon_B$ and~$z$ are independent and both have mean zero, \eqref{eq:brightness_var_prop} follows immediately.

For the standard diffusion model, apply~$B_n$ to
\begin{align*}
\bm{x}_t^\mathrm{std} = \sqrt{\bar{\alpha}_t}\bm{x}_0 + \sqrt{1-\bar{\alpha}_t} \bm{\epsilon}_0,
\end{align*}
which yields \eqref{eq:brightness_std}; the variance formula \eqref{eq:brightness_var_std} follows in the same way.

\subsection{Proof of the $v$-prediction proposition} \label{apdx:proof_prop_vpred}
Following \cite{salimans2022progressive}, from \eqref{eq:xt_x0_ep0_f}, we have
\begin{align}
\bm{x}_0 & = \sqrt{\bar{\alpha}_t} \bm{x}_t - \sqrt{1 - \bar{\alpha}_t} \bm{v}_t, \label{eq:x0_xt_vt}  \\
\bm{v}_t & = \sqrt{\bar{\alpha}_t} \left(\sigma_0 \bm{\epsilon}_0 + \psi_t \bm{\xi} \right) - \sqrt{1 - \bar{\alpha}_t} \bm{x}_0, \notag
\end{align}
where $\bm{v}_t = \frac{d \bm{x}_t}{d \omega_t}$ and $\omega_t$ is the angle satisfying $\cos(\omega_t)=\sqrt{\bar{\alpha}_t}, \sin(\omega_t)=\sqrt{1 - \bar{\alpha}_t}$.
Substituting \eqref{eq:x0_xt_vt} into \eqref{eq:tilde_mu_x0} yields
\begin{align*}
\tilde{\mu}(\bm{x}_t, \bm{x}_0, \bm{\xi}) & = \frac{\sqrt{\alpha_t} (1-\bar{\alpha}_{t-1})}{1 - \bar{\alpha}_t} \bm{x}_t + \frac{(1 - \alpha_t)\sqrt{\bar{\alpha}_{t-1}}}{1 - \bar{\alpha}_t} \left(\sqrt{\bar{\alpha}_t} \bm{x}_t - \sqrt{1 - \bar{\alpha}_t} \bm{v}_t \right) + \nu_t \bm{\xi}  \\
& = \frac{\sqrt{\alpha}_t(1-\bar{\alpha}_{t-1}) + (1 - \alpha_t)\sqrt{\bar{\alpha}_{t-1} \bar{\alpha}_t}}{1 - \bar{\alpha_t}} \bm{x}_t - \frac{(1 - \alpha_t) \sqrt{\bar{\alpha}_{t-1}}}{\sqrt{1 - \bar{\alpha}_t}} \left(\bm{v}_t - \frac{\sqrt{1 - \bar{\alpha}_t}}{(1 - \alpha_t) \sqrt{\bar{\alpha}_{t-1}}} \nu_t  \bm{\xi}\right).
\end{align*}
Since $\nu_t=0$ under the noise-matching strategy, this simplifies to
\begin{align*}
\tilde{\mu}(\bm{x}_t, \bm{x}_0, \bm{\xi}) & = \frac{\sqrt{\alpha}_t(1-\bar{\alpha}_{t-1}) + (1 - \alpha_t)\sqrt{\bar{\alpha}_{t-1} \bar{\alpha}_t}}{1 - \bar{\alpha_t}} \bm{x}_t - \frac{(1 - \alpha_t) \sqrt{\bar{\alpha}_{t-1}}}{\sqrt{1 - \bar{\alpha}_t}} \bm{v}_t  \\
& = \frac{\sqrt{\alpha}_t(1-\bar{\alpha}_{t-1}) + (1 - \alpha_t)\sqrt{\bar{\alpha}_{t-1} \bar{\alpha}_t}}{1 - \bar{\alpha_t}} \bm{x}_t - \frac{(1 - \alpha_t) \sqrt{\bar{\alpha}_{t-1}}}{\sqrt{1 - \bar{\alpha}_t}} \left(\sqrt{\bar{\alpha}_t}\left(\sigma_0 \bm{\epsilon}_0 + \psi_t \bm{\xi} \right) - \sqrt{1 - \bar{\alpha}_t} \bm{x}_0\right).
\end{align*}
Therefore, if we parameterize $\mu_\theta(\bm{x}_t, t)$ as
\begin{align*}
\mu_\theta(\bm{x}_t, t) & = \frac{\sqrt{\alpha}_t(1-\bar{\alpha}_{t-1}) + (1 - \alpha_t)\sqrt{\bar{\alpha}_{t-1} \bar{\alpha}_t}}{1 - \bar{\alpha_t}} \bm{x}_t - \frac{(1 - \alpha_t) \sqrt{\bar{\alpha}_{t-1}}}{\sqrt{1 - \bar{\alpha}_t}} v_\theta(\bm{x}_t, t),
\end{align*}
then the KL divergence in~$\mathcal{L}_3$ can be written as follows for $t=2,\ldots,T$:
\begin{align*}
& D_{\text{KL}}\left(q(\bm{x}_{t-1} | \bm{x}_{t}, \bm{x}_0, \bm{\xi}) \ || \ p_\theta(\bm{x}_{t-1}|\bm{x}_t) \right)  \\
& = \frac{1}{2 \sigma_t^2} \mathbb{E}_{q(\bm{x}_{t-1} | \bm{x}_{t}, \bm{x}_0, \bm{\xi})} \left[\left\|\tilde{\mu}(\bm{x}_t, \bm{x}_0, \bm{\xi}) - \mu_\theta(\bm{x}_t, t)\right\|^2 \right] + C_1  \\
& = \frac{\bar{\alpha}_{t-1} (1-\alpha_t)^2}{2 \sigma_t^2 (1-\bar{\alpha}_t)} \mathbb{E}_{\mathcal{N}(\bm{\epsilon}_0 | 0, I)} \left[\left\|\sqrt{\bar{\alpha}_t}\left(\sigma_0 \bm{\epsilon}_0 + \psi_t  \bm{\xi} \right) - \sqrt{1 - \bar{\alpha}_t} \bm{x}_0 - v_\theta(\bm{x}_t, t)\right\|^2 \right] + C_3,
\end{align*}
where $C_3$ is a constant independent of $\theta$.

For~$\mathcal{L}_1$ in \eqref{eq:elbo_prop}, we have
\begin{align*}
& \mathbb{E}_{q(\bm{x}_{1} | \bm{x}_0, \bm{\xi})}\left[\log p_\theta(\bm{x}_0 | \bm{x}_1) \right]  \\
& = \mathbb{E}_{q(\bm{x}_{1} | \bm{x}_0, \bm{\xi})}\left[ \log \mathcal{N}\left(\bm{x}_0 \ \middle| \ \mu_\theta(\bm{x}_1, 1), \sigma_1^2 I\right) \right]  \\
& = - \frac{1}{2 \sigma_1^2} \mathbb{E}_{q(\bm{x}_{1} | \bm{x}_0, \bm{\xi})}\left[\left\|\bm{x}_0 - \mu_\theta(\bm{x}_1, 1)\right\|^2 \right] + C_2  \\
& = - \frac{1}{2 \sigma_1^2} \mathbb{E}_{q(\bm{x}_{1} | \bm{x}_0, \bm{\xi})}\left[ \left\|\bm{x}_0 - \frac{\sqrt{\alpha}_1(1-\bar{\alpha}_{0}) + (1 - \alpha_1)\sqrt{\bar{\alpha}_{0} \bar{\alpha}_1}}{1 - \bar{\alpha}_1} \bm{x}_1 + \frac{(1 - \alpha_1) \sqrt{\bar{\alpha}_{0}}}{\sqrt{1 - \bar{\alpha}_1}} v_\theta(\bm{x}_1, 1)\right\|^2 \right] + C_2 \\
& = - \frac{1}{2 \sigma_1^2} \mathbb{E}_{q(\bm{x}_{1} | \bm{x}_0, \bm{\xi})}\left[ \left\|\bm{x}_0 - \sqrt{\alpha_1} \bm{x}_1 + \sqrt{1 - \alpha}_1 v_\theta(\bm{x}_1, 1)\right\|^2 \right] + C_2  \\
& = - \frac{1}{2 \sigma_1^2} \mathbb{E}_{q(\bm{x}_{1} | \bm{x}_0, \bm{\xi})}\left[ \left\|-\sqrt{1 - \alpha_1} \bm{v}_1 + \sqrt{1 - \alpha}_1 v_\theta(\bm{x}_1, 1)\right\|^2 \right] + C_2  \\
& = - \frac{1 - \alpha_1}{2 \sigma_1^2} \mathbb{E}_{q(\bm{x}_{1} | \bm{x}_0, \bm{\xi})}\left[ \left\| \bm{v}_1 -  v_\theta(\bm{x}_1, 1)\right\|^2 \right] + C_2  \\
& = - \frac{1 - \alpha_1}{2 \sigma_1^2} \mathbb{E}_{\mathcal{N}(\bm{\epsilon}_0 | 0, I)}\left[ \left\| \sqrt{\alpha_1} \left(\sigma_0 \bm{\epsilon}_0 + \psi_1 \bm{\xi} \right) - \sqrt{1 - \alpha_1} \bm{x}_0 - v_\theta(\bm{x}_1, 1)\right\|^2 \right] + C_2.
\end{align*}
Therefore, under the noise-matching strategy, the training loss that maximizes the evidence lower bound in~\eqref{eq:elbo_prop} with respect to~$\theta$ under the~$v$-prediction formulation is
\begin{align*}
\ell^v(\theta) = \mathbb{E}_{q(\bm{\xi}), \mathcal{U}(t | 1,T), \mathcal{N}(\bm{\epsilon}_0 | 0, I)}  \left[\frac{\bar{\alpha}_{t-1} (1-\alpha_t)^2}{2 \sigma_t^2 (1-\bar{\alpha}_t)} \left\|\sqrt{\bar{\alpha}_t}\left(\sigma_0 \bm{\epsilon}_0 + \psi_t \bm{\xi} \right) - \sqrt{1 - \bar{\alpha}_t} \bm{x}_0 - v_\theta(\bm{x}_t, t)\right\|^2 \right],
\end{align*}
where~$\bm{x}_t$ is given by \eqref{eq:xt_x0_ep0_f}.


\section{Additional experimental results}

\subsection{Training with data scaling} \label{apdx:exp_data_scaling}
The Base model was trained on the Cylinder dataset with dimensionality~$n=200$ using data scaling with scaling parameter~$\rho$. Specifically,~$\rho$ was set to one of~$0.7, 0.8, 0.9, 1.1, 1.2$, or~$1.3$ (note that~$\rho=1.0$ corresponds to the case without data scaling).
The 1WD and MMD values during training for each configuration are shown in Figure~\ref{fig:1wd_mmd_rescale}. For comparison, the figure also includes the results for the Base model without data scaling~($\rho=1.0$) and the proposed model~($\sigma_c^2=1.0$).
Applying data scaling with~$\rho=1.1$ to the Base model yields smaller 1WD and MMD values than the case without scaling~($\rho=1.0$).
However, the proposed model achieves even smaller 1WD and MMD values.
\begin{figure}[tbhp]
\centering
\includegraphics[width=0.8\linewidth,clip]{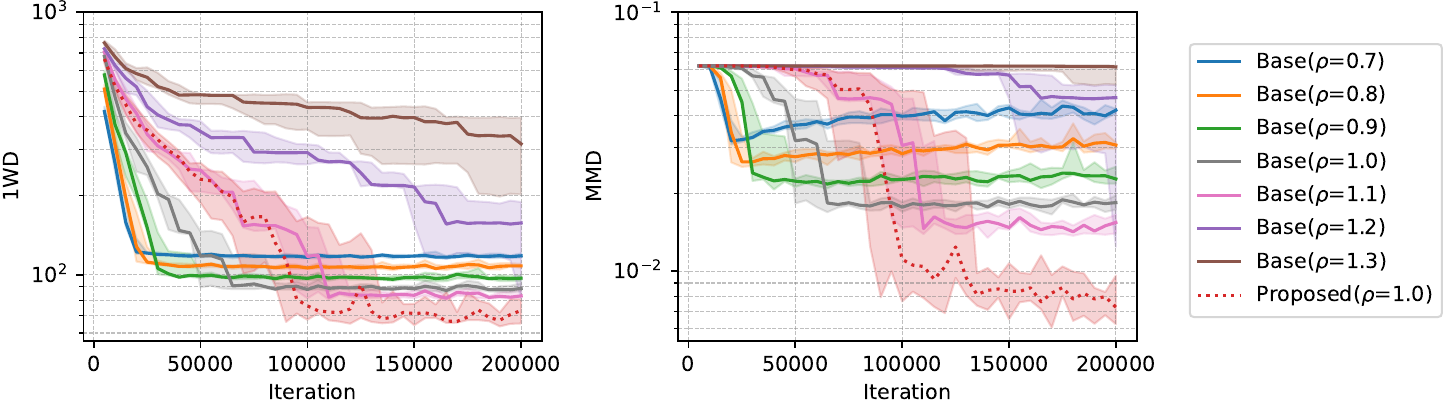}
\caption{Changes in 1WD and MMD during the training of the Base model~($n=200$) with data scaling using various scaling parameters~$\rho$. For comparison, the results of the proposed model without data scaling~($\rho=1.0$) and the Proposed model~($\sigma_c^2=1.0$) are also included.}
\label{fig:1wd_mmd_rescale}
\end{figure}

Next, for each Base model trained with data scaling, we generated~5000 samples and compared the distribution of their average brightness~$B_n(\bm{x}_0)$ with that of the test dataset. The results are shown in Figure~\ref{fig:histgram_scale}. As the figure shows, applying data scaling to the Cylinder dataset~($n=200$) does not substantially change the distribution of~$B_n(\bm{x}_0)$ in the generated samples. This again suggests that data scaling alone does not resolve the difficulty of generating data with extreme average brightness.
\begin{figure}[tbhp]
\centering
\includegraphics[width=1.0\linewidth,clip]{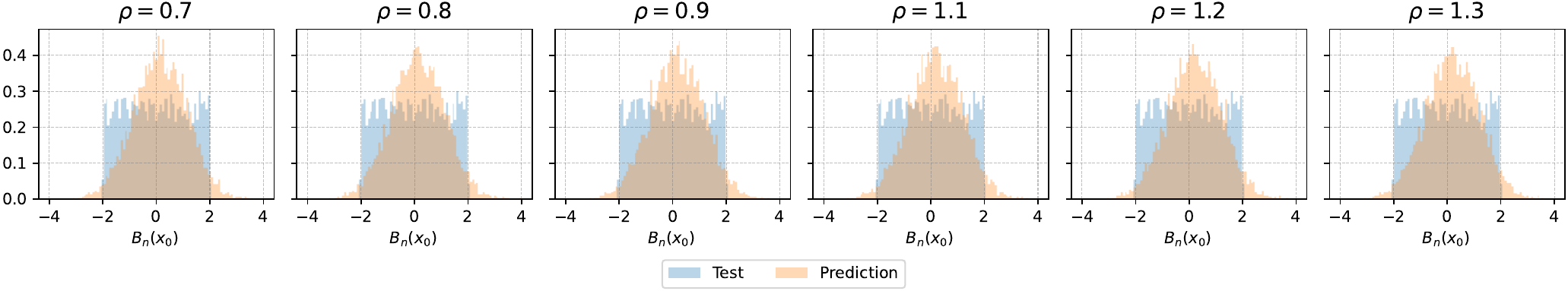}
\caption{Comparison of average brightness~$B_n(\bm{x}_0)$ distributions of the Base model~($n=200$) with data scaling using various scaling parameters~$\rho$.}
\label{fig:histgram_scale}
\end{figure}

\subsubsection{Comparison of average brightness distributions for $v$-prediction models} \label{apdx:brightness_vpred}
Figure~\ref{fig:histgram_vpred} compares the distributions of~$B_n(\bm{x}_0)$ for samples generated by each $v$-prediction model with that of the test dataset.
For each~$n$, the top, middle, and bottom rows correspond to the Base model, the Offset noise model~($\sigma_c^2=0.1$), and the proposed model~($\sigma_c^2=1.0$), respectively. In each case, the model used is the one obtained after the final training step.
\begin{figure}[tbhp]
\centering
\includegraphics[width=1.0\linewidth,clip]{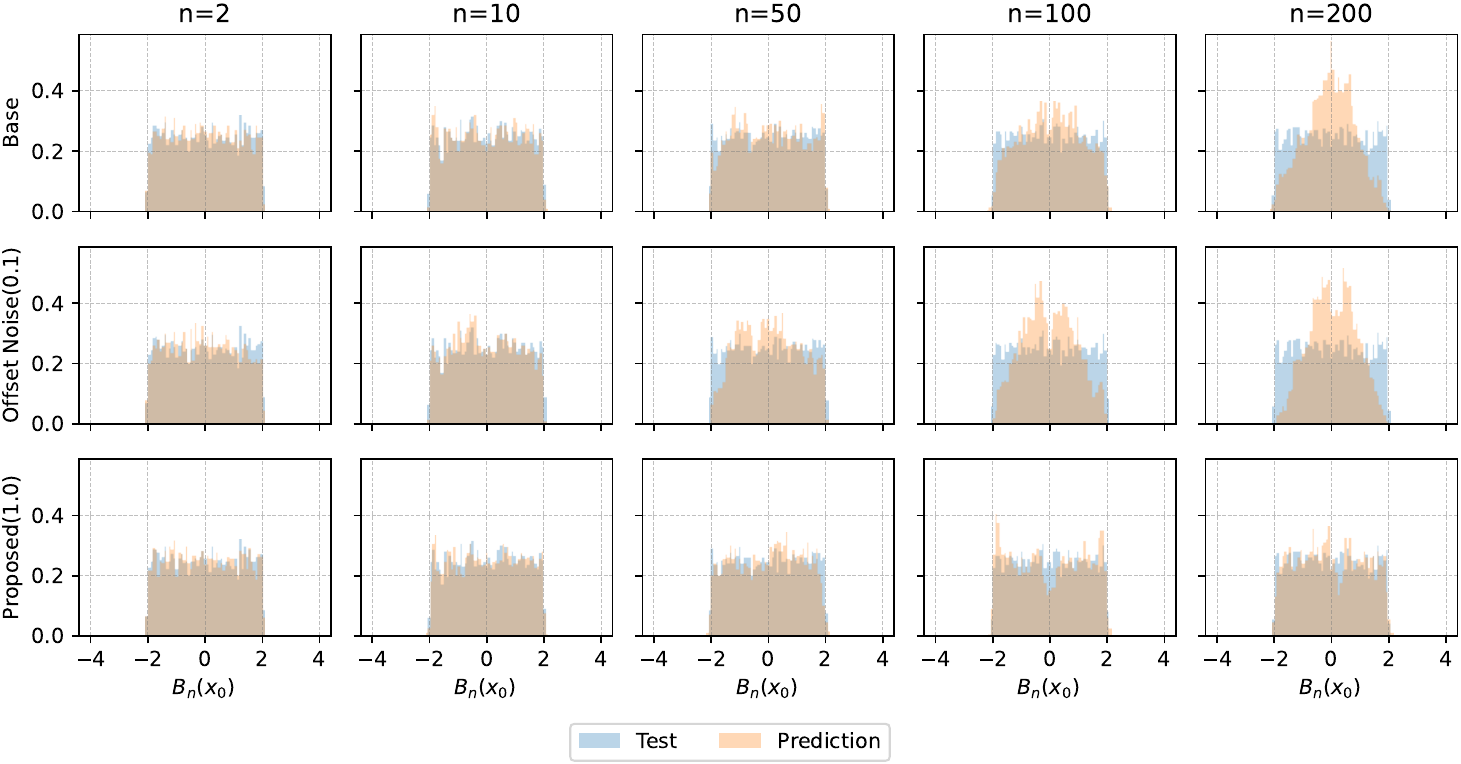}
\caption{Comparison of distributions of average brightness~$B_n(\bm{x}_0)$ between the test data and the generated data using $v$-prediction models.}
\label{fig:histgram_vpred}
\end{figure}

\section{Python code for generating the Cylinder dataset} \label{apdx:cylinder_code}
Figure~\ref{fig:cylinder_code} shows the Python code for generating the Cylinder dataset.
\begin{figure}[htb]
\begin{lstlisting}[language=Python]
import torch

def cylinder_dataset(size: int, dim: int, r: float = 0.5, top_center: float = 2.0):
"""
Generate a cylinder-shaped dataset in n-dimensional space.

Args:
    size: Number of samples to generate.
    dim: Dimensionality of the space.
    r: Radius of the cylinder (relative to the norm of a vector of ones).
    top_center: The center of the top of the cylinder.

Returns:
    Tensor containing the generated cylinder dataset.
"""

# Create a vector of all ones, which will define the cylinder's axis direction.
vec_ones = torch.ones(dim)

# Adjust the radius relative to the dimensionality using L2 norm.
adjusted_r = r * vec_ones.norm(p=2)

# Generate random unit vectors orthogonal to vec_ones(the cylinder's axis).
vec_ortho = torch.randn(size, dim)
vec_ortho = vec_ortho - vec_ortho.mm(vec_ones[:, None]) / dim * vec_ones
vec_ortho = vec_ortho / vec_ortho.norm(p=2, dim=1)[:, None]

# Scale the orthogonal vectors by random radii within the cylinder's radius.
vec_ortho = vec_ortho * torch.rand(size).mul(adjusted_r)[:, None]

# Scale vec_ones to random heights within [-top_center, top_center].
vec_ones = vec_ones * torch.rand(size).mul(2 * top_center).sub(top_center)[:, None]

# Combine the axis and orthogonal components to form the final dataset.
data = vec_ones + vec_ortho

return data
\end{lstlisting}
\caption{Python code for generating the Cylinder dataset.}
\label{fig:cylinder_code}
\end{figure}

\bibliographystyle{plainnat}
\bibliography{references}

\end{document}